%% file: taf_tpami.tex
\documentclass[10pt,journal,compsoc]{IEEEtran}

\usepackage{amsmath,amssymb,amsfonts,amsthm}
\usepackage{thmtools}
\usepackage{thm-restate}
\usepackage{algorithm}
\usepackage{algorithmic}
\usepackage{array}
\usepackage{booktabs}
\usepackage{multirow}
\usepackage{enumitem}
\usepackage{wrapfig}
\usepackage{makecell}
\usepackage{bbm}
\usepackage{graphicx}
\usepackage[caption=false,font=footnotesize]{subfig}
\usepackage{adjustbox}
\usepackage{bbding}
\usepackage[table]{xcolor}
\usepackage{textcomp}
\usepackage{stfloats}
\usepackage{url}
\usepackage{dsfont}
\usepackage{hyperref}

\input{math_commands.tex}

\definecolor{mine}{RGB}{205,232,248}

\newtheorem{theorem}{Theorem}[section]
\newtheorem{proposition}{Proposition}[section]

\newtheorem{corollary}{Corollary}[section]
\theoremstyle{definition}
\newtheorem{definition}{Definition}[section]
\newtheorem{assumption}{Assumption}[section]

\theoremstyle{remark}
\newtheorem{remark}{Remark}[section]

\renewcommand{\eqref}[1]{\textup{(\ref{#1})}}
\newcommand{\bestmean}[1]{\textbf{#1}}
\newcommand{\bestcell}[2]{\cellcolor{green!15}\bestmean{#1}\!\pm\!#2}
\newcommand{\bestplain}[1]{\cellcolor{green!15}\bestmean{#1}}
\newcommand{\tighttab}{
  \setlength{\tabcolsep}{4pt}
  \renewcommand{\arraystretch}{0.98}
}

\providecommand{\citep}[1]{\cite{#1}}
\providecommand{\citet}[1]{\cite{#1}}

\title{Target-Aligned Fusion for Decision-Sequence Learning under Dynamics Shift}

\author{Guojian Wang, Quinson Hon, Xuyang Chen, and Lin Zhao%
}

\begin{document}
\maketitle
\begin{abstract}
External trajectories can improve offline decision-sequence learning, but dynamics shift may make some source subsequences inconsistent with the target environment. We study how to fuse such trajectories with limited target data for Decision Transformer learning under dynamics shift. We propose Target-Aligned Fusion (TAF), a principled framework that derives source-data fusion from a target-domain Bellman-risk criterion. Our analysis bounds this risk by two measurable data-alignment quantities: $\Delta_m$, the state-structure mismatch of retained fragments, and $\Delta_w$, the weighted transport cost from source to target transitions. This decomposition yields a gate--then--weight rule: source fragments are first filtered by target-side state-structure alignment, and retained transitions are then reweighted by local target feasibility. We instantiate this principle as TAF-DT, which uses maximum mean discrepancy (MMD) for fragment selection, optimal transport for feasibility-aware weighting, and the resulting fused law for advantage-token relabeling and Q-regularized Transformer training. Across gravity, kinematic, and morphology shifts on D4RL-style control tasks, TAF-DT achieves the strongest aggregate performance against strong offline RL and sequence-model baselines and produces more stable stitch-junction sequence semantics. Overall, these results indicate that aligning external trajectories to target-domain structure and feasibility is a practical way to exploit source data under dynamics shift.

\end{abstract}


\section{Introduction}\label{sec:introduction}

\IEEEPARstart{L}{earning} from fused source and target sequences under distribution shift requires more than enlarging the dataset. Additional source trajectories can improve coverage, but they can also corrupt the local semantics that a target-domain sequence learner relies on. This broader issue is especially visible in dynamics-shifted offline control with models such as the Decision Transformer (DT), which condition future actions on patterns observed in logged trajectories \citep{Levine2020OfflineRL,chen2021decision,janner2021trajectory}. When source and target data differ through the transition dynamics, imported source fragments may no longer remain locally compatible with the target environment: state structures drift across domains, return conditioning becomes brittle under dynamics-induced return-distribution and effective-horizon changes, and actions feasible under source dynamics can become implausible at stitch junctions. Naively fused sequences can therefore improve nominal coverage while degrading target-side reliability.

We study this problem through target-domain decision-sequence learning with Decision Transformers in a dynamics-shifted offline control setting. Existing methods filter, align, reweight, or transport source data through support-aware selection, representation matching, stationary-distribution regularization, or optimal-transport (OT) alignment \citep{liu2024beyond, wen2024contrastive, xue2023state, lyu2025cross}. These strategies are useful, but they are usually introduced as preprocessing or regularization heuristics. What remains missing is a target-side criterion that determines \emph{which} imported source sequences should be retained and \emph{how} retained transitions should be weighted so that downstream learning remains reliable on the target domain. In sequence settings, this question is sharper than generic distribution matching: imported subsequences must be structurally stitchable to target trajectories, while retained local transitions must remain feasible under target-side dynamics.

\begin{figure*}
\centering
\includegraphics[width=0.8\linewidth]{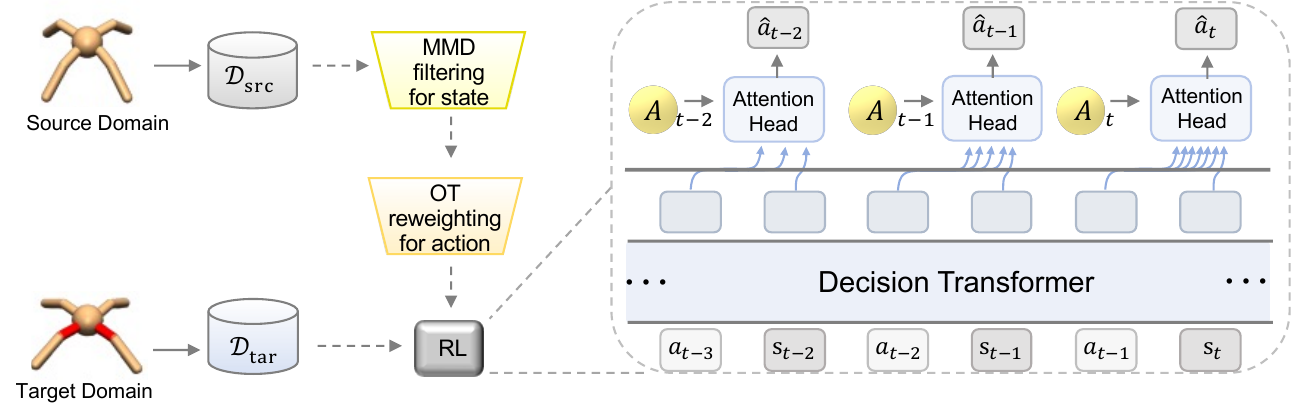}
\caption{
An overview of TAF and its Decision-Transformer instantiation TAF-DT as a target-aligned sequence-fusion pipeline for dynamics-shifted offline control. Source fragments are first screened by an MMD-based state-structure gate and then reweighted by an OT-based action-feasibility score so that the critic learned on the fused data remains close to Bellman-consistent on the target domain. The resulting feasibility-weighted trajectories are then consumed through advantage-conditioned tokens and weighted $Q$-regularized Transformer training to produce stable actions $\hat a_t$ under target-domain dynamics.
}
\label{fig:training}
\end{figure*}

We address this gap through a \emph{target-aligned sequence-fusion} view. Rather than treating source selection, weighting, and downstream sequence learning as separate design choices, we start from a single criterion: imported source trajectories should be admitted only insofar as the fused distribution preserves target-side Bellman reliability. This yields a \emph{target Bellman-risk matching} formulation and, in turn, a simple \emph{gate--then--weight} principle: source fragments should first satisfy target-side state-structure compatibility, and retained transitions should then be reweighted according to target-side action and transition feasibility. We refer to the resulting framework as Target-Aligned Fusion (TAF), and denote its Decision-Transformer-based instantiation by TAF-DT. In TAF-DT, the fused distribution is used consistently for critic fitting, advantage-conditioned tokenization, and weighted $Q$-regularized Transformer training. The formal guarantee is deliberately kept at the Bellman-risk level: the stitchability radii justify the fusion rule and its tractable surrogate, while the downstream return gains are evaluated empirically under controlled dynamics shifts.

Our contributions are threefold:
\begin{itemize}
    \item \textbf{Problem formulation.} Motivated by distribution-shifted decision-sequence learning, we formulate source--target fusion for dynamics-shifted offline control with Decision Transformers as a \emph{target Bellman-risk matching} problem. This turns source-data fusion from heuristic preprocessing into a target-side critic-reliability criterion.
    \item \textbf{Variational principle and theory.} We prove that the measurable stitchability radii control target Bellman risk and induce a tractable variational \emph{gate--then--weight} surrogate. This provides a Bellman-risk-level justification for MMD-based fragment gating and OT-based feasibility weighting without requiring a strong global performance-guarantee assumption.
    \item \textbf{Concrete instantiation and evidence.} We instantiate this principle in TAF-DT by first constructing a target-aligned fused distribution through MMD-based fragment gating and OT-based feasibility weighting, and then carrying that same fused distribution into advantage-conditioned tokenization and weighted $Q$-regularized Transformer training. Empirically, TAF-DT achieves strong aggregate returns and stable stitch-junction sequence semantics across gravity, kinematic, and morphology shifts.
\end{itemize}

The remainder of the paper reviews related work, introduces the offline-control setup and notation, develops the target-risk formulation and variational surrogate, presents the concrete TAF-DT instantiation, and then reports experiments, limitations, and conclusions.

\section{Related Work}\label{sec:related}
TAF connects three neighboring areas: offline reinforcement learning, source--target data fusion under structural mismatch, and sequence modelling for decision making. These areas address different parts of the setting considered here. Offline RL studies reliable learning from fixed logged data; source--target fusion studies how shifted data can be aligned or selected; and decision-sequence models study how trajectories can be represented and optimized by sequence learners. TAF differs by making the \emph{fused training law} the primary object: source data are not merely regularized at the policy level or aligned as a preprocessing step, but are selected and weighted according to a target Bellman-risk matching principle that is then used consistently for critic fitting, sequence construction, and Transformer training.

\textbf{Offline reinforcement learning.}
Offline RL aims to learn effective policies from previously collected data without further environment interaction \citep{Levine2020OfflineRL}. A central difficulty is distributional shift between the learned policy and the behavior policy that generated the data, which can cause value overestimation and unreliable extrapolation. Representative approaches address this issue through behavior constraints, as in BCQ and BEAR \citep{Fujimoto2019BCQ,Kumar2019BEAR}, pessimistic or conservative value learning, as in CQL and IQL \citep{Kumar2020CQL,kostrikov2022iql}, model-based pessimism such as MOReL \citep{kidambi2020morel}, or behavior-regularized policy learning such as TD3$+$BC \citep{Fujimoto2021TD3BC}. These methods generally assume that the offline dataset has already been specified and focus on making value or policy learning robust on that dataset. TAF addresses a complementary earlier question: when target data are scarce and shifted source trajectories are available, how should the training distribution itself be formed before offline learning takes place?

\textbf{Source--target data fusion under dynamics and structural shift.}
A second line of work studies how source data can be used in a target domain when the two domains differ in dynamics, morphology, sensing, or other structural factors. Prior methods use representation alignment, MMD-style distribution matching, optimal transport, source selection, or domain-adaptive generation to make imported data more compatible with the target domain \citep{liu2024beyond,wen2024contrastive,Lyu2024ODRL,lyu2025cross,Li2024DiffStitch,le2025dmc}. TAF shares the goal of exploiting shifted source data, but differs in the criterion used to decide what should be imported. Rather than treating alignment or reweighting as standalone preprocessing, TAF chooses the fused law through a target-domain Bellman-risk view. This yields a gate--then--weight structure in which fragment-level state-structure compatibility and transition-level target feasibility play distinct roles.

\textbf{Sequence modelling for decision making.}
Decision Transformers and related trajectory models cast offline control as conditional sequence modelling over returns, states, and actions \citep{chen2021decision,janner2021trajectory}. Subsequent work improves sequence-based control by incorporating critic information, value-aware conditioning, or advantage-aware tokens to mitigate the brittleness of pure return-to-go prompting \citep{gao2024act,Wang2023CGDT,hu2024q}. RADT further studies off-dynamics adaptation mainly through return-space adjustment \citep{wang2024return}. TAF-DT is complementary to these model-level advances. It couples upstream transition-space fusion with downstream sequence-level stabilization: the Bellman-risk-matched fusion distribution is used for critic fitting, advantage-conditioned tokenization, and weighted $Q$-regularized Transformer training.

\section{Problem Setup and Preliminaries}\label{sec:preliminaries}

We now make the target-aligned sequence-fusion question concrete in a dynamics-shifted offline control setting and introduce the notation used throughout the paper. Our focus is adaptation from a rich source dataset collected under dynamics $P_S$ to a target environment with different dynamics $P_T$, where only limited target data are available.

\subsection{Terminology}
In this paper, \emph{stitchability} means that imported source fragments can be used in target-domain sequence learning without creating target-inconsistent local transitions. A \emph{stitch junction} is a boundary created by concatenating two fragments or by crossing between a retained source fragment and target-domain context in either direction. \emph{State-structure compatibility} refers to latent state-pattern alignment between a source fragment and target data, later measured by MMD. \emph{Action feasibility} refers to local action--transition compatibility with target dynamics, later measured by OT-based transition costs.

\subsection{MDP and shift setting.}
We consider two infinite-horizon Markov Decision Processes (MDPs), the source domain $\mathcal{M}_S := (\mathcal{S},\mathcal{A},P_S,r,\gamma,\rho_0)$ and the \emph{target} domain $\mathcal{M}_{T} := (\mathcal{S},\mathcal{A},P_{T},r,\gamma,\rho_0)$. The two domains share $\mathcal{S}$, $\mathcal{A}$, the bounded reward $r:\mathcal{S}\times\mathcal{A}\to\mathbb{R}$, discount factor $\gamma\in[0,1)$, and initial distribution $\rho_0$, but differ in their transition kernels $P_{S}\neq P_{T}$. Throughout the paper, plain $P_S,P_T$ denote MDP transition kernels, whereas blackboard-bold symbols such as $\mathbb P_S,\mathbb P_T$, and $\mathbb P_{\mathrm{mix}}^{I,w}$ denote data or training distributions. For any MDP $\mathcal{M}$ and policy $\pi$, let $d_{\mathcal{M}}^{\pi}(s) := (1-\gamma)\sum_{t=0}^{\infty}\gamma^t P_{\mathcal{M}}^{\pi}(s\!\mid\!t)$ denote the normalized discounted state occupancy, and define $J_{\mathcal{M}}(\pi):=\mathbb{E}_{(s,a)\sim\nu_{\mathcal{M}}^{\pi}}[r(s,a)]$.

\textbf{Offline datasets and DT-specific difficulty.}
Let $\mathcal{D}_{\mathrm{src}}$ be an offline dataset from $\mathcal{M}_{S}$ and $\mathcal{D}_{\mathrm{tar}}$ a much smaller dataset from $\mathcal{M}_{T}$. We aim to learn a policy $\pi^\star$ that maximizes $J_T(\pi)$ without online interaction with $\mathcal{M}_T$. The core challenge is the dynamics shift $P_S\neq P_T$, which is especially harmful to DT-style policies that rely on token-level continuity in RTG, state, and action: source and target state manifolds misalign, RTG becomes unstable under dynamics-induced return-distribution and effective-horizon changes, and actions feasible under $P_S$ can be implausible under $P_T$. Using only $\mathcal{D}_{\mathrm{src}}$ extrapolates invalid next tokens under $P_T$, while $\mathcal{D}_{\mathrm{tar}}$ alone lacks coverage.

\subsection{Expectile regression and advantage estimation.}
For a response $Y\in\mathbb{R}$ and covariates $X\in\mathcal{X}$, the $\zeta$-expectile regression function ($\zeta\in(0,1)$) is the map $m_\zeta:\mathcal{X}\to\mathbb{R}$ that minimizes $\mathbb{E}[\rho_\zeta(Y-m_\zeta(X))]$, where $\rho_\zeta(u)=|\zeta-\mathbf{1}\{u<0\}|u^2$. It is unique under mild integrability by strict convexity; $\zeta=\tfrac{1}{2}$ recovers the conditional mean, while larger $\zeta$ emphasizes the upper tail.

\subsection{Maximum mean discrepancy.}\label{sec:latent-mmd}
Let $f_\phi:\mathcal S\to\mathcal Z$ be the encoder that maps each state $s\in\mathcal S$ to its latent representation $z=f_\phi(s)\in\mathcal Z$, and let $k:\mathcal Z\times\mathcal Z\to\mathbb R$ be a bounded positive-definite kernel with RKHS $\mathcal H$. For two probability measures $\mu$ and $\nu$ on $\mathcal Z$, which can be viewed as the latent distributions induced by applying $f_\phi$ to the corresponding state distributions, the maximum mean discrepancy induced by $k$ is defined as
\begin{equation}\label{eq:mmd_def}
\mathrm{MMD}_k(\mu,\nu)
:=
\sup_{\|h\|_{\mathcal H}\le 1}
\big|\mathbb E_{z\sim\mu}[h(z)]-\mathbb E_{z\sim\nu}[h(z)]\big|,
\end{equation}
where $z$ denotes a latent representation and the supremum is taken over all RKHS functions $h\in\mathcal H$ with unit norm. This quantity provides the latent-space discrepancy measure used later to define the structural radius $\Delta_m$.

\section{Target-Risk Formulation for Target-Aligned Sequence Fusion}
\label{sec:method}
This section develops TAF from the viewpoint of target-domain Bellman risk under source--target fusion. Building on the preceding discussion, we do not begin by prescribing a particular filtering, reweighting, or sequence-learning mechanism. Instead, we ask a more basic question: which fused training distribution induces a critic that remains most Bellman-consistent with target-domain dynamics? This perspective first identifies the ideal fusion rule at the population level, then yields a variational gate--then--weight principle, and finally leads to its concrete TAF-DT instantiation. The sequence-construction and optimization details are deferred to Section~\ref{sec:technical}.

\subsection{Bilevel fusion objective and target Bellman risk}
\label{sec:risk-matching}
Given scarce target data and abundant but shifted source data, the issue is not merely whether source data should be used, but how it should be fused so that learning on the resulting distribution remains reliable for the target domain. We therefore cast source--target fusion as a \emph{target Bellman-risk matching} problem: among admissible fusion laws, we seek the one whose induced critic best preserves target-domain Bellman consistency. The gate--then--weight structure of TAF will then emerge as a tractable variational surrogate of this population objective.

\noindent\textbf{Problem (Two-level source-data selection and weighting for target-aligned offline learning).}
Let $\mathbb P_T$ and $\mathbb P_S$ denote the target and source transition distributions, respectively. Let $u=(s,a,r,s')$ denote a transition and $\tau=(s_1,a_1,r_1,\dots,s_n,a_n,r_n)$ a trajectory fragment ($n\ge 1$). We seek a state fragment-level gate $I(\cdot)\in\{0,1\}$ and a transition-level weight $w(\cdot)\ge 0$ such that the induced weighted source distribution
$
\mathbb P_S^{I,w}(u)\;\propto\;I(\tau^S)\,w(u)\,\mathbb P_S(u)
$
combines with $\mathbb P_T$ to form the fusion distribution
\[
\mathbb P_{\mathrm{mix}}^{I,w}=(1-\beta)\,\mathbb P_T+\beta\,\mathbb P_S^{I,w},\qquad \beta\in[0,1].
\]
Here the gate removes source fragments whose state structure is poorly matched to the target domain, i.e., fragments that are not readily stitchable to target trajectories. The weight function then modulates retained source transitions according to their action-level feasibility. This family covers pure target training ($\beta=0$), filtering-only fusion, and fusion with both filtering and reweighting. Throughout Section~\ref{sec:risk-matching} and the generic appendix proofs, we use the abstract notation $(I,w)$ and the induced laws $\mathbb P_S^{I,w}$ and $\mathbb P_{\mathrm{mix}}^{I,w}$. When instantiated by TAF, we set $I=I_m$ and write $\mathbb P_S^{I_m,w}$ and $\mathbb P_{\mathrm{mix}}^{I_m,w}$.

To expose our target Bellman-risk matching objective, we first introduce the target-domain Bellman residual of the induced critic. Specifically,
\begin{align}
\delta_T^{V}(s)
&:=
\mathbb E_{\mathbb P_T}\!\big[r(s,a)+\gamma V^{I,w}(s')-V^{I,w}(s)\,\big|\,s\big],
\label{eq:def-deltaV-risk}
\end{align}
Here, \(V^{I,w}\) denotes the critic induced by the fusion distribution determined by the fragment-level gate \(I\) and transition-level weight \(w\); thus \(\delta_T^{V}(s)\) measures the conditional target-domain Bellman residual of this induced critic at state \(s\).

Let $\mu_T^{s}:=\operatorname{Marg}_{s}(\mathbb P_T)$ denote the state marginal of the target data law. With $\tau^S\in\mathcal D_{\rm src}$ denoting a source trajectory fragment and $u\in\tau^S$ a transition within that fragment, we define the following bilevel objective 
\begin{subequations}
\label{eq:ideal-bellman-objective}
\begin{align}
V^{I,w}
&\in
\arg\min_{V}\,\mathcal L_{\mathrm{critic}}\!\left(V;\mathbb P_{\mathrm{mix}}^{I,w}\right),
\label{eq:ideal-bellman-objective-inner}
\\
(I^{\star},w^{\star})
&\in
\arg\min_{I,w}\,
\mathfrak{R}_{\mathrm{Bell}}\!\left(V^{I,w};\mathbb P_T\right),
\label{eq:ideal-bellman-objective-outer}
\end{align}
\end{subequations}
where $\mathcal L_{\mathrm{critic}}(V;\mathbb P_{\mathrm{mix}}^{I,w})$ denotes the critic-learning objective under the fusion distribution, and the target Bellman risk $\mathfrak{R}_{\mathrm{Bell}}$ of the induced critic is
\begin{align}
\mathfrak{R}_{\mathrm{Bell}}\!\left(V^{I,w};\mathbb P_T\right)
:=\|\delta_T^V\|_{1,\mu_T^{s}}
=\mathbb E_{s\sim\mu_T^s}\!\left[\left|\delta_T^V(s)\right|\right].
\label{eq:bellman-risk-def}
\end{align}
For brevity, we use the reduced notation
\[
\mathfrak{R}_{\mathrm{Bell}}\!\big({I,w}\big)
\;:=\;
\mathfrak{R}_{\mathrm{Bell}}\!\left(V^{I,w};\mathbb P_T\right),
\]
which emphasizes that the outer objective depends on $(I,w)$ only through the critic induced by training on $\mathbb P_{\mathrm{mix}}^{I,w}$. Under this view, the outer fusion rule of Equation~\eqref{eq:ideal-bellman-objective} determines which source information is retained and how strongly it contributes, while the inner learner fits value functions on the induced fusion distribution. Equation~\eqref{eq:ideal-bellman-objective} therefore identifies the ideal fusion rule as the one whose induced critic minimizes target-domain Bellman inconsistency.

\subsection{From stitchability radii and approximation errors to a target Bellman-risk matching bound}

Direct optimization of Eq.~\eqref{eq:ideal-bellman-objective} is intractable because the outer problem depends on the critic produced by the inner learning dynamics. Instead of solving that bilevel program directly, we upper-bound the target Bellman risk by quantities that depend only on the fused data distribution and mild regularity properties of the learned critic. To obtain this bound, we first introduce two data-side discrepancy measures and the critic-side error quantities used in the bound. 

\subsubsection{Stitchability radii and critic-side errors}

The two data-side discrepancy measures are the latent-space MMD radius and the $1$-Wasserstein transport radius. The first captures the worst residual state-structure mismatch among retained source fragments, and the second captures the remaining action--transition mismatch between the weighted source law and the target law. These are exactly the two mismatches that TAF controls through fragment gating and within-support reweighting.
\begin{definition}[Stitchability radii]\label{def:radii}
For a gate--weight pair $(I,w)$, define the MMD-based state-structure radius
\[
\Delta_m:=\sup_{\tau^S:I(\tau^S)=1} d^m(\tau^S),
\]
where $d^m(\tau^S)= \mathbb E_{\tau^T}\mathrm{MMD}_k(\tau^S,\tau^T)$ is a fragment-level state-structure discrepancy, computed in TAF by the latent-space MMD score in Eq.~\eqref{eq:mmd_def}. Define also the $1$-Wasserstein transport radius
\[
\Delta_w:=W_1\!\big(\mathbb P_S^{I,w},\,\mathbb P_T\big),
\]
where $W_1$ denotes the $1$-Wasserstein (Earth Mover's) distance~\cite{villani2003topics,villani2008optimal} on the transition space under the ground cost $C(\cdot,\cdot)$. For readability, we suppress the explicit dependence of $(\Delta_m,\Delta_w)$ on $(I,w)$ when no ambiguity arises, and write $\Delta_m(I)$ or $\Delta_w(I,w)$ only when that dependence needs to be emphasized.
\end{definition}

These two radii are not introduced as separate heuristics. They arise naturally in the proof of the transfer bound: the fragment gate leaves a worst-case residual structural discrepancy among retained fragments, captured by the MMD radius $\Delta_m$, while the weighting step leaves a residual source-to-target transport discrepancy, captured by the Wasserstein radius $\Delta_w$. Together they are the two data-side quantities through which Bellman-risk control depends on source--target stitchability.

The two stitchability radii alone do not determine target Bellman risk; we also need to measure how accurately the value critic induced by the fused distribution fits its own training law. We therefore define the mixed-domain value Bellman residual
\begin{align*}
\delta_{\mathrm{mix}}^{V}(s)
:=
\mathbb E_{\mathbb P_{\mathrm{mix}}^{I,w}}\!\big[
 r(s,a)+\gamma V^{I,w}(s')-V^{I,w}(s)\,\big|\,s\big],
\end{align*}
and denote its worst-case magnitude by
\begin{align*}
\varepsilon_V
:=\sup_{s\in\mathcal S}\big|\delta_{\mathrm{mix}}^{V}(s)\big|.
\end{align*}
This mixed-domain approximation error measures the value-residual fitting error of the critic learned from $\mathbb P_{\mathrm{mix}}^{I,w}$ before any target-side transfer gap is accounted for.

\subsubsection{Target Bellman-risk matching bound under stitchability radii}

With these radii and error terms in place, the intractable outer objective can be replaced by a tractable target-risk bound that depends on the structural radius $\Delta_m(I)$, the transport radius $\Delta_w(I,w)$, and the critic-side errors summarized by $\varepsilon_V$ and $\varepsilon_H$. However, to lift these data-side discrepancies to Bellman transfer, we assume that the learned value is approximately constant within encoder fibres and that the shared reward is bounded and transport-regular.
\begin{assumption}[Approximate fiber-constancy of $V$]\label{ass:fiber-constancy}
Let $f_\phi$ map inputs $s$ (either full states in a fully observed MDP, or observation histories in a partially observed setting) from $\mathcal S$ to latent codes in $\mathcal Z$. Then the value varies little within each encoder fiber: there exists $\varepsilon_H\ge 0$ such that
\[
\sup\big\{|V(s)-V(\tilde s)|: s,\tilde s\in\mathcal S,\ f_\phi(s)=f_\phi(\tilde s)\big\}
\le
\varepsilon_H.
\]
\end{assumption}

\begin{assumption}[Bounded and transport-regular reward]\label{ass:bounded}
The shared reward is bounded as $|r|\le R_{\max}$. When viewed as a function of the transition tuple $u$, $r(s,a)$ is $L_r$-Lipschitz under the transition metric $\rho$.
\end{assumption}

The next theorem makes the reduction precise. In particular, the two radii control the target-domain Bellman error of the value critic learned from the fused distribution.
\begin{theorem}[Target Bellman-risk matching bound under stitchability radii]\label{thm:bellman-risk}
Under Assumptions~\ref{ass:fiber-constancy} and \ref{ass:bounded}, together with the latent MMD setup in Section~\ref{sec:latent-mmd}, let $V$ be learned under $\mathbb P_{\mathrm{mix}}^{I,w}$ with the fixed mixture coefficient $\beta\in[0,1]$. Then there exist constants $R_1,R_2>0$, depending only on $R_{\max}$, $L_r$, and encoder/kernel bounds, such that
\begin{align}
\mathfrak{R}_{\mathrm{Bell}}({I,w})
\le
\varepsilon_V+\beta\,(R_1\,\Delta_m+R_2\,\Delta_w)+4\beta\,\varepsilon_H,
\label{eq:target-risk-V}
\end{align}
\end{theorem}

\noindent\textbf{Interpretation.}
The bound separates the target Bellman risk into a transfer term, $\beta(R_1\Delta_m+R_2\Delta_w)$, and an approximation term, $\varepsilon_V+4\beta\varepsilon_H$. Consequently, controlling only one radius of $\Delta_m$ or $\Delta_w$ is insufficient: filtering without reweighting can still leave action-feasibility mismatch, while reweighting without structural filtering can assign mass to fragments that are already misaligned in state structure. Theorem~\ref{thm:bellman-risk} therefore provides the theoretical basis for TAF's coupled gate--then--weight design. The complete proof is given in Appendix~\ref{sec:proof_of_target_risk}.


\subsection{Variational surrogate and sequential gate--then--weight optimizer}
\label{sec:vari-sur}
Theorem~\ref{thm:bellman-risk} identifies the data-dependent quantities that must be controlled for target Bellman-risk matching. The terms $\varepsilon_V$ and $4\beta\varepsilon_H$ reflect critic approximation and representation errors, and are not directly optimized by the data-side gate--weight stage. In contrast, $\Delta_m(I)$ and $\Delta_w(I,w)$ depend on the retained source support and its weights. We therefore use the bound to define a tractable data-side surrogate: first control the fragment-level structural radius, and then assign weights according to target-side transport feasibility on the retained support.

\textbf{From the transport radius to samplewise feasibility costs.}
Setting $\lambda_m:=\beta R_1$ and $\lambda_w:=\beta R_2$, Theorem~\ref{thm:bellman-risk} gives, for every admissible gate--weight pair,
\begin{equation}
\label{eq:bellman-upper-prevar}
\mathfrak{R}_{\rm Bell}(I,w)
\le
\varepsilon_V + \lambda_m\,\Delta_m(I) + \lambda_w\,\Delta_w(I,w) + 4\beta\,\varepsilon_H .
\end{equation}
The remaining step is to convert the distribution-level radius $\Delta_w(I,w)=W_1(\mathbb P_S^{I,w},\mathbb P_T)$ into transition-level feasibility costs for optimizing $w$. Let the retained weighted source law be
\[
\mathbb P_S^{I,w}=\sum_{i\in\mathcal G(I)} w_i\delta_{u_i},
\]
where $\mathcal G(I)$ denotes the retained transition support. Any feasible coupling $\Pi\in\Gamma(\mathbb P_S^{I,w},\mathbb P_T)$ admits a row-wise disintegration
\[
\Pi({\rm d}u,{\rm d}v)
=
\sum_{i\in\mathcal G(I)}
w_i\,\delta_{u_i}({\rm d}u)\,\kappa_i({\rm d}v),
\quad
\sum_{i\in\mathcal G(I)} w_i\kappa_i=\mathbb P_T,
\]
where $\kappa_i$ is the target-side conditional law paired with source transition $u_i$. Therefore the transport cost decomposes into row costs
\begin{equation}
\label{eq:samplewise-cost-def}
c_i^{\kappa}
:=
\int C(u_i,v)\,{\rm d}\kappa_i(v),
\quad
\int C\,{\rm d}\Pi
=
\sum_{i\in\mathcal G(I)} w_i c_i^{\kappa}.
\end{equation}
Since $W_1$ is the minimum over feasible couplings, any such row decomposition gives the samplewise majorization
\begin{equation}
\label{eq:samplewise-majorizer}
\Delta_w(I,w)
=
W_1(\mathbb P_S^{I,w},\mathbb P_T)
\le
\sum_{i\in\mathcal G(I)} w_i c_i^{\kappa}.
\end{equation}
This majorization is the desired bridge: after the MMD gate, TAF fixes a reference target-side disintegration on the retained support and uses its induced row costs as stable per-transition feasibility scores. In practice, a product coupling yields a conservative choice, while the empirical implementation below uses a sharper reference OT-row construction. Denoting the resulting fixed cost by $c_i$, and adding an entropy regularizer relative to the reference law $u_I$, we obtain the data-side variational surrogate
\begin{equation}
\label{eq:variational-surrogate}
\mathcal J_{\tau_{\mathrm{ent}}}(I,w)
=
\lambda_m\,\Delta_m(I)
+\lambda_w\!\sum_{i\in\mathcal G(I)} w_i c_i
+\tau_{\mathrm{ent}}\,\mathrm{KL}(w\|u_I),
\end{equation}
where $\tau_{\mathrm{ent}}>0$ and $w\in\Delta(\mathcal G(I)):=\{w_i\ge 0,\ \sum_{i\in\mathcal G(I)} w_i=1\}$. The KL term is not an additional consequence of the Bellman bound itself; it is introduced at the surrogate level to prevent the linear transport objective from collapsing onto a few lowest-cost transitions. Thus the surrogate preserves the risk-matching interpretation of the two radii while converting the within-support optimization into a stable soft-selection problem.

\textbf{Structural gate.}
Under a fixed retention budget, the first stage of Equation~\eqref{eq:variational-surrogate} is to retain the fragments that minimize the structural radius $\Delta_m(I)$. This yields the following threshold rule.
\begin{proposition}[Optimal threshold gate under a retention budget]
\label{prop:optimal-gate}
Let $d_i^m:=d^m(\tau_i^S)$ be the state-structure discrepancy of source fragment $i$, and fix a retention budget $n_{\mathrm{keep}}\in\{1,\dots,N\}$. Consider the structural gate problem
\begin{equation}
\label{eq:gate-subproblem}
\min_{I\in\{0,1\}^N}\ \Delta_m(I)
\qquad
\text{s.t.}\quad
\sum_{i=1}^N I_i = n_{\mathrm{keep}}.
\end{equation}
Then any optimizer of Eq.~\eqref{eq:gate-subproblem} retains exactly the $n_{\mathrm{keep}}$ fragments with the smallest values of $d_i^m$. Equivalently, there exists a threshold $q_{n_{\mathrm{keep}}}$ given by the $n_{\mathrm{keep}}$-th order statistic of $\{d_i^m\}_{i=1}^N$ such that
\[
I_i^{\star}=\mathbf 1\!\big(d_i^m\le q_{n_{\mathrm{keep}}}\big),
\]
up to arbitrary tie-breaking at the threshold.
\end{proposition}
\noindent The full proof is deferred to Appendix~\ref{sec:proof_variational}.

\textbf{OT-based soft reweighting.}
Once the support has been fixed by the structural gate, the second stage solves an entropy-regularized transport-feasibility problem on that support. This produces a Gibbs reweighting that favors target-compatible transitions while retaining support diversity.
\begin{proposition}[Optimal entropic reweighting on the gated support]
\label{prop:optimal-weight}
Fix a gate $I$ and let $\mathcal G(I)$ be its retained support. Consider the within-support transport-alignment problem
\begin{equation}
\label{eq:weight-subproblem}
\min_{w\in\Delta(\mathcal G(I))}
\Big\{
\lambda_w\!\sum_{i\in\mathcal G(I)} w_i c_i
+ \tau_{\mathrm{ent}}\,\mathrm{KL}(w\|u_I)
\Big\}.
\end{equation}
Then Eq.~\eqref{eq:weight-subproblem} has the unique optimizer
\begin{equation}
\label{eq:gibbs-weights}
w_i^{\star}
=
\frac{u_I(i)\,\exp(-\lambda_w c_i/\tau_{\mathrm{ent}})}
{\sum_{j\in\mathcal G(I)} u_I(j)\,\exp(-\lambda_w c_j/\tau_{\mathrm{ent}})}.
\end{equation}
Equivalently, if $d_i^w:=-c_i$ and $\eta_w:=\lambda_w/\tau_{\mathrm{ent}}$, then $w_i^{\star}\propto u_I(i)\exp(\eta_w d_i^w)$, which is the exponential OT-feasibility weighting used by TAF.
\end{proposition}
\noindent The detailed derivation is given in Appendix~\ref{sec:proof_variational}.

\textbf{Sequential gate--then--weight characterization.}
The two propositions combine into a sequential optimizer for the Bellman-risk surrogate: first minimize the structural radius under the retention budget, and then solve the regularized transport-feasibility problem on the retained support.
\begin{theorem}[Variational characterization of TAF gate--then--weight fusion]
\label{thm:variational-dfdt}
Fix a retention budget $n_{\mathrm{keep}}$, let $\mathcal I_{n}:=\{I\in\{0,1\}^N:\sum_i I_i=n_{\mathrm{keep}}\}$, and define the sequential surrogate program
\begin{equation}
\label{eq:sequential-surrogate-program}
\begin{aligned}
I^{\star}
&\in
\operatorname{argmin}_{I\in\mathcal I_n}
\Delta_m(I),
\\
w^{\star}
&\in
\operatorname{argmin}_{w\in\Delta(\mathcal G(I^{\star}))}
\Big\{
\lambda_w\!\sum_{i\in\mathcal G(I^{\star})} w_i c_i
+\tau_{\mathrm{ent}}\,\mathrm{KL}(w\|u_{I^{\star}})
\Big\}.
\end{aligned}
\end{equation}
Then:
\begin{enumerate}[leftmargin=1.5em]
    \item $I^{\star}$ is the threshold MMD gate of Proposition~\ref{prop:optimal-gate}, i.e., the top-$n_{\mathrm{keep}}$ fragments with smallest state-structure discrepancy.
    \item $w^{\star}$ is the Gibbs OT-feasibility reweighting of Proposition~\ref{prop:optimal-weight} on the retained support.
    \item The induced pair $(I^{\star},w^{\star})$ satisfies
    \begin{equation}
    \label{eq:variational-risk-bound}
    \begin{aligned}
    \mathfrak{R}_{\mathrm{Bell}}\!\big(I^{\star},w^{\star}\big)
    &\le \varepsilon_V + \lambda_m\,\Delta_m(I^{\star}) \\
    &\quad + \lambda_w\!\sum_{i\in\mathcal G(I^{\star})} w_i^{\star} c_i
          + 4\beta\,\varepsilon_H .
    \end{aligned}
    \end{equation}
\end{enumerate}
\end{theorem}
Full proofs are provided in Appendix~\ref{sec:proof_variational}. We next give the implementation-level computations for the gate and the within-support weights.

\subsection{Concrete computation of the TAF gate and weights}
\label{sec:tech-gateweight}
This subsection gives the concrete computations used by TAF: latent-space MMD scores for ranking source fragments and empirical OT costs for assigning Gibbs weights on the retained support.

\subsubsection{MMD-based fragment selection}
\label{sec:mmd-select}
TAF ranks each source fragment by a latent-space MMD score relative to target fragments. Let $\{\tau_i^S\}_{i=1}^N$ denote the source fragments, with
\[
\tau_i^S=(s_{i,1}^S,a_{i,1}^S,r_{i,1}^S,\dots,s_{i,\ell_i}^S,a_{i,\ell_i}^S,r_{i,\ell_i}^S),
\]
and let $f_\phi$ be the shared encoder. Writing $z_{i,t}^S=f_\phi(s_{i,t}^S)$ and, for a target fragment $\tau^T$ of length $\ell$, $z_t^T=f_\phi(s_t^T)$, the fragment-level latent MMD is
\begin{equation}\label{eq:mmd_dist}
\begin{aligned}
\mathrm{MMD}_k^2(\tau_i^S,\tau^T)
= &\frac{1}{\ell_i^2}\sum_{p,q=1}^{\ell_i} k(z_{i,p}^S,z_{i,q}^S)
+ \frac{1}{\ell^2}\sum_{p,q=1}^{\ell} k(z_p^T,z_q^T) \\
&- \frac{2}{\ell_i\ell}\sum_{p=1}^{\ell_i}\sum_{q=1}^{\ell} k(z_{i,p}^S,z_q^T),
\end{aligned}
\end{equation}
with $k$ an RBF kernel on the latent space. TAF then scores each source fragment by
\begin{equation}
\label{eq:mmd_cost}
d_i^m
:= d^m(\tau_i^S)
=
\mathbb E_{\tau^T\sim \mathbb P_T}
\Big[
\mathrm{MMD}_k(\tau_i^S,\tau^T)
\Big].
\end{equation}
Given a retention budget $n_{\mathrm{keep}}$ (or equivalently a keep ratio $\xi$ with $n_{\mathrm{keep}}=\lfloor \xi N\rfloor$), TAF implements the gate $I_m$ by thresholding these scores at the $n_{\mathrm{keep}}$-th order statistic, retaining the fragments with the smallest MMD values. By Proposition~\ref{prop:optimal-gate}, this is exactly the structural-gate optimizer induced by the variational surrogate.

\subsubsection{OT-based action feasibility and weights}
After the MMD gate fixes the retained support, the OT step assigns a soft feasibility weight to each retained transition.  Let $\mathcal G(I_m)=\{u_i\}_{i=1}^{n}$ denote the retained source transitions and let $\{v_j\}_{j=1}^{M}$ denote target transitions, where $u_i=(s_i^S,a_i^S,r_i^S,{s_i^S}^{\prime})$ and $v_j=(s_j^T,a_j^T,r_j^T,{s_j^T}^{\prime})$.  In implementation, both are embedded as the same concatenated transition feature $x(u)=[s,a,r,s']$.  The ground cost used by the OT solver is the cosine cost, following OTDF~\citep{Lyu2025OTDF}:
\begin{equation}
\label{eq:ot-cos-cost}
C(u_i,v_j)
=
1-\frac{\langle x(u_i),x(v_j)\rangle}
{\|x(u_i)\|_2\|x(v_j)\|_2+\epsilon_C}.
\end{equation}

We solve one entropic reference OT problem with ground cost $C(u_i,v_j)$ between the uniform law on the retained support, $u_I(i)=1/n$, and the empirical target law, $q_j=1/M$; solving this problem gives the reference plan $\widehat\gamma^0$. Note that  $\widehat\gamma^0$ is not recomputed for every possible downstream source weight; it is used only to turn each retained source transition into a fixed row-wise feasibility cost.  Specifically, for each row we normalize the transported target mass and average the cosine costs:
\begin{equation}
\label{eq:emp-ot-cost}
\hat\pi_{ij}:=\frac{\widehat\gamma^0_{ij}}{\sum_{k=1}^{M}\widehat\gamma^0_{ik}},
\qquad
\hat c_i:=\sum_{j=1}^{M}\hat\pi_{ij}\,C(u_i,v_j).
\end{equation}
Here $\hat\pi_{ij}\ge0$ and $\sum_j\hat\pi_{ij}=1$, so $\hat c_i$ is the average target-side transport cost seen from source transition $u_i$.  Smaller values indicate better local action--transition compatibility with the target data.

\textbf{Cost calibration.}
Raw OT costs can have different numerical ranges across tasks.  We therefore calibrate them on the retained support using their minimum $a_I$ and range $s_I$:
\begin{equation}
\label{eq:ot-cost-calib-main}
a_I:=\min_{\ell\in\mathcal G(I_m)}\hat c_\ell,
\ \ 
s_I:=\max_{\ell\in\mathcal G(I_m)}\hat c_\ell-a_I,
\ \ 
\widetilde c_i:=\frac{\hat c_i-a_I}{s_I+\epsilon}.
\end{equation}
Thus $a_I$ recenters the best retained transition to zero, while $s_I$ rescales the cost spread before exponentiation.  This calibration preserves the ordering induced by $\hat c_i$ and mainly controls the effective temperature of the Gibbs weights.  We finally set the feasibility score and source weight as
\begin{align}
 d_i^w &:= -\widetilde c_i, \label{eq:ot-dist}\\
 w_i
 &=
 \frac{u_I(i)\exp(\eta_w d_i^w)}
 {\sum_{\ell\in\mathcal G(I_m)}u_I(\ell)\exp(\eta_w d_\ell^w)},
 \qquad i\in\mathcal G(I_m). \label{eq:ot-weight}
\end{align}
When $u_I$ is uniform, it is absorbed into the normalizer.  In this way, OT is used as a soft feasibility weighting rule on the MMD-retained support, rather than as a second hard filtering stage.

After gating and reweighting, TAF forms the empirical weighted source law whose mass on each retained transition $u_i\in\mathcal G(I_m)$ is proportional to $u_I(i)\exp(\eta_w d_i^w)$. Equivalently, in density notation,
\[
\mathbb P_S^{I_m,w}(u)\propto I_m(\tau^S)\,\exp\!\big(\eta_w d^w(u)\big)\,\mathbb P_S(u),
\]
where the uniform reference factor $u_I$ is absorbed into the normalizer in our implementation. The final mixed law is
\begin{equation}
\label{eq:pmix-def}
\mathbb P_{\mathrm{mix}}^{I_m,w}=(1-\beta)\,\mathbb P_T+\beta\,\mathbb P_S^{I_m,w},
\qquad \beta\in[0,1].
\end{equation}
This mixed law is the output of the gate--then--weight stage and is used throughout the remainder of TAF-DT.

\section{Technical Instantiation of TAF-DT}
\label{sec:technical}
The preceding bound is algorithmic rather than merely diagnostic. It identifies the data-side quantities optimized by TAF-DT: \(\Delta_m(I_m)\) induces the MMD fragment gate, and \(\Delta_w(I_m,w)\) induces OT-based feasibility weights. The resulting fused law \(\mathbb P_{\mathrm{mix}}^{I_m,w}\) is then used throughout critic fitting, advantage-token relabeling, and Transformer policy training. Hence, the theory justifies the construction of the training law.

\textbf{Difference from related components.}
OTDF~\cite{Lyu2025OTDF} aligns transition distributions with OT~\citep{Lyu2025OTDF}, but it does not first impose a fragment-level structural gate or propagate a single risk-matched fusion law into sequence-token relabeling. QT and ACT improve DT training through value regularization or advantage conditioning on a fixed dataset~\citep{gao2024act,hu2024q}, whereas TAF-DT first constructs \(\mathbb P_{\mathrm{mix}}^{I_m,w}\) by Bellman-risk-driven fusion and then uses advantage tokens and \(Q\)-regularization as downstream consumers of that law. Thus, the distinction is not the use of MMD, OT, or critic-conditioned tokens alone, but their coupling through one target-aligned training distribution.

\subsection{Sequence-aware trajectory stitching with advantage-conditioned tokenization}
\label{sec:adv-tokens}
Having fixed the fused law $\mathbb P_{\mathrm{mix}}^{I_m,w}$ in Eq.~\eqref{eq:pmix-def}, we next explain how TAF-DT turns samples from this distribution into DT-compatible sequences. Standard return-to-go (RTG) tokens are brittle under cross-domain stitching: dynamics changes can alter return distributions and effective horizons, so RTG can jump even when the fused data are already target-aligned. TAF-DT therefore replaces RTG with advantage-conditioned tokens computed from auxiliary critics trained on the same fused distribution.

For empirical implementation, it is convenient to associate each transition $u$ with a scalar sample weight
\[
\omega(u)=
\begin{cases}
1, & u\sim \mathbb P_T,\\
I_i\exp(\eta_w d_i^w), & u=u_i\in\mathcal G(I_m),
\end{cases}
\]
so that expectations under $\mathbb P_{\mathrm{mix}}^{I_m,w}$ can be realized as weighted empirical averages over target data and gated source data. The coefficient $\beta$ sets the overall target--source mixing proportion, whereas $\omega(u)$ encodes source-side selection and weighting: $I_i$ implements fragment gating and $\exp(\eta_w d_i^w)$ reweights retained source transitions by OT-based feasibility.

\textbf{Auxiliary critics for token relabeling.}
Let $V_\varphi$ and $Q_\psi$ denote auxiliary value functions used only to construct the token labels. They are trained under $\mathbb P_{\mathrm{mix}}^{I_m,w}$ so that the resulting advantages are comparable across target data and retained weighted source data:
\begin{align}
\mathcal{L}_\varphi
&=
\mathbb{E}_{(s,a,r,s')\sim \mathbb P_{\mathrm{mix}}^{I_m,w}}
\Big[
\rho_\zeta\!\big(r+\gamma V_\varphi(s')-V_\varphi(s)\big)
\Big],
\label{V_loss}
\\
\mathcal{L}_\psi
&=
\mathbb{E}_{(s,a,r,s')\sim \mathbb P_{\mathrm{mix}}^{I_m,w}}
\Big[
\rho_{\frac{1}{2}}\!\big(r+\gamma V_\varphi(s')-Q_\psi(s,a)\big)
\Big],
\label{Q_loss}
\end{align}
where $\rho_\zeta$ is the expectile loss from Section~\ref{sec:preliminaries}. The first loss fits a shared value baseline and the second fits an auxiliary action-value function against the same bootstrap target. Because both are trained on the same fused law that appears in Section~\ref{sec:risk-matching}, the induced advantages inherit the same gate--then--weight bias toward target-feasible behavior.

\textbf{Advantage-conditioned tokens.}
For each transition in a stitched sequence, we define
\[
A_t:=Q_\psi(s_t,a_t)-V_\varphi(s_t).
\]
This advantage token replaces the raw RTG signal. For a relabeled window $x_{t:t+K-1}$, TAF-DT therefore uses the sequence
\begin{equation}
\label{eq:adv-token-seq}
\bigl(s_t,a_t,A_t,\ldots,s_{t+K-1},a_{t+K-1},A_{t+K-1}\bigr)
\end{equation}
as the DT context. Compared with RTG, $A_t$ is local, critic-normalized, and substantially less sensitive to dynamics-induced return-distribution and effective-horizon changes, which makes stitched windows more stable near source--target junctions.

\textbf{Command-token replacement at inference.}
During training, the advantage tokens in Eq.~\eqref{eq:adv-token-seq} are observed from data. During inference, they are replaced by a command network $C_\omega$ that predicts an advantage-consistent command token from the current state. This preserves the same conditioning format at test time without requiring future rewards. The training details of $C_\omega$ are given in Section~\ref{sec:alg-details}.

\subsection{Target-aligned transformer training via weighted Q-regularization}
\label{sec:q-reg}
We now train the final critics and Transformer policy on the same fused law $\mathbb P_{\mathrm{mix}}^{I_m,w}$. To avoid confusion with the auxiliary token critics $V_\varphi$ and $Q_\psi$ from Section~\ref{sec:adv-tokens}, we denote the training-time twin critics by $Q_{\phi_1}$ and $Q_{\phi_2}$. Every loss below is written as an expectation under $\mathbb P_{\mathrm{mix}}^{I_m,w}$; equivalently, target samples are used directly and source samples are admitted by the gate $I_i$ and weighted by $\exp(\eta_w d_i^w)$. Algs.~\ref{alg:train} and~\ref{alg:infer} summarize the full training and inference procedures.

\textbf{Twin-critic update.}
Let $x_{t-K+1:t}$ be an advantage-conditioned context window from Eq.~\eqref{eq:adv-token-seq}. The training critics are updated by minimizing
\begin{align}
\mathcal L_{Q_{\phi_j}}
=
\mathbb E_{\mathbb P_{\mathrm{mix}}^{I_m,w}}
\left[
\frac{1}{K-1}\sum_{i=t-K+1}^{t-1}
\bigl(\hat Q_i-Q_{\phi_j}(s_i,a_i)\bigr)^2
\right],
\label{eq:q_loss}
\end{align}
where
\[
\hat Q_i
=
\sum_{\ell=i}^{t-1}\gamma^{\ell-i}r_\ell
+\gamma^{t-i}\min_{j\in\{1,2\}}Q_{\phi'_j}(s_t,\hat a_t),
\]
and
\[ 
\hat a_t\sim \pi_{\theta'}(\cdot\mid x_{t-K+1:t}),
\]
and $Q_{\phi'_j}$ and $\pi_{\theta'}$ are the target critics and target policy. This is the TD fitting step whose Bellman consistency motivated the gate--then--weight fusion in Section~\ref{sec:risk-matching}.

\textbf{Weighted DT loss.}
Using the same relabeled window $x_{t-K+1:t}$, the policy is trained by sequence modeling under the fused law:
\begin{align}
\mathcal L_{\mathrm{DT}}
=
\mathbb E_{\mathbb P_{\mathrm{mix}}^{I_m,w}}
\left[
\frac{1}{K}\sum_{i=t-K+1}^{t}
\|a_i-\pi_\theta(x_{t-K+1:t})_i\|_2^2
\right].
\label{eq:dt_loss}
\end{align}

\textbf{$Q$-regularized policy objective.}
To favor high-value actions while preserving sequence fidelity, TAF-DT augments the DT loss with a critic-guidance term:
\begin{equation}
\label{eq:pi_loss}
\begin{aligned}
\mathcal L_{\pi}
=
&\ \mathcal L_{\mathrm{DT}}(\theta)
\ + \eta_{\rm reg}\,\mathcal L_{\rm reg}(\pi_\theta,\pi_{\rm tar})\\
&-\frac{\alpha}{K}\,
\mathbb E_{\mathbb P_{\mathrm{mix}}^{I_m,w}}
\left[\sum_{i=t-K+1}^{t}
Q_\phi\!\big(s_i,\pi_\theta(x_{t-K+1:t})_i\big)
\right],
\end{aligned}
\end{equation}
where $Q_\phi:=\min_{j\in\{1,2\}}Q_{\phi_j}$, $\pi_{\rm tar}$ is the target-domain behavior policy induced by $\mathcal D_{\rm tar}$, $\mathcal L_{\rm reg}$ is a similarity penalty, and $\alpha,\eta_{\rm reg}\ge 0$ are the critic-guidance and trust-region regularization coefficients. Thus, the same fused distribution $\mathbb P_{\mathrm{mix}}^{I_m,w}$ governs all three stages of TAF-DT: it trains the auxiliary critics used for token relabeling, the twin critics used for TD guidance, and the final Transformer policy.

\section{Experiments}\label{sec:exp}
We evaluate whether the Bellman-risk-driven fusion principle translates into practical gains under gravity, kinematic, and morphology shifts. The evidence is organized in three layers: main results on target-domain return, diagnostics at stitch junctions, and ablations of the key design choices. Detailed task-wise gravity results, efficiency analyses, broader implementation details, and additional hyperparameter ablations are deferred to the supplementary material after the references.

\subsection{Experimental setup}
\textbf{Tasks and datasets.}
We evaluate policy adaptation under three dynamics shifts, \emph{gravity}, \emph{kinematics}, and \emph{morphology}, on four MuJoCo tasks (HalfCheetah, Hopper, Walker2d, Ant) in OpenAI Gym \cite{brockman2016openai}. Gravity scales the magnitude of $g$; kinematics constrains joint ranges; morphology changes link dimensions. We adopt the configurations of \cite{lyu2025cross}. The setting is dynamics-shifted cross-domain offline control: abundant source data but scarce target data from shifted environments. Sources are D4RL “-v2” datasets (medium, medium-replay, medium-expert) \citep{fu2020d4rl}; targets are the D4RL-style datasets of \cite{lyu2025cross} (medium / medium-expert / expert), each with $5{,}000$ transitions. This low-data regime, known to challenge standard offline RL \cite{liu2024beyond,wen2024contrastive,lyu2025cross}, yields 108 tasks across the three shift families and serves as a controlled testbed for distribution-shifted decision-sequence learning.

\textbf{Baselines.} 
We compare TAF-DT with strong offline RL: IQL \citep{kostrikov2022iql} (expectile value regression with advantage-weighted policy), and sequence-modelling baselines for cross-domain adaptation: DT \citep{chen2021decision} (return-to-go sequence model), QT \citep{hu2024q} (value-aware DT), and a DADT variant \citep{kim2022dynamics} with dynamics-aware tokenisation but no filtering. We also include recent cross-domain methods: DARA \citep{liu2022dara}, IGDF \citep{wen2024contrastive}, and OTDF \citep{lyu2025cross}, covering reward reweighting, representation filtering, and OT-based data fusion. For all baselines, we use the authors’ recommended hyperparameters and code, modifying only the dataset and environment identifiers.

\textbf{Evaluation protocol.}
We adopt the cross-domain setup in Sec.~\ref{sec:preliminaries}, using abundant D4RL source logs (\emph{medium}, \emph{medium-replay}, \emph{medium-expert}) and scarce target logs (\emph{medium}, \emph{medium-expert}, \emph{expert}) collected under gravity, kinematic, and morphology shifts. We report normalized target-domain returns (\emph{mean~$\pm$~std.}) over \emph{five} seeds while sweeping all \{\emph{source quality}\}$\times$\{\emph{target quality}\} pairs ($3{\times}3$) across \emph{halfcheetah}, \emph{hopper}, \emph{walker2d}, and \emph{ant}. All methods train offline on the prescribed source and target logs. The normalized-score definition, environment details, and additional implementation settings are given in the supplementary material.

\subsection{Main Results: Returns Under Dynamics Shifts}

We train TAF-DT for 100k gradient updates with five random seeds and report normalized target-domain scores. Detailed results for morphology and kinematic shifts are shown in Tables \ref{tab:morph-shift} and \ref{tab:kinematic-shift}; for gravity shifts, we report a compact summary in Table~\ref{tab:shift-summary} and place the full 36-task table in the supplementary material.

\begin{table*}[ht]
\caption{\textbf{Performance comparison of cross-domain offline RL algorithms under morphology shifts.} $\text{half}\!=\!\text{halfcheetah}$, $\text{hopp}\!=\!\text{hopper}$, $\text{walk}\!=\!\text{walker2d}$, $\text{m}\!=\!\text{medium}$, $\text{r}\!=\!\text{replay}$, $\text{e}\!=\!\text{expert}$. The ‘Target’ column indicates target-domain offline data quality. We report \emph{normalized} target-domain performance (\emph{mean $\pm$ std.}) across source qualities \{\emph{medium, medium-replay, medium-expert}\} and target qualities \{\emph{medium, medium-expert, expert}\}, averaged over \textbf{five} seeds; best per row is highlighted.}
\label{tab:morph-shift}
\centering
\small
\tighttab
\begin{tabular}{
    @{}ll
    | >{$}c<{$} 
    | >{$}c<{$} >{$}c<{$} >{$}c<{$} 
    | >{$}c<{$} >{$}c<{$} >{$}c<{$} 
    | >{$}c<{$} 
    @{}}
\toprule
Source & Target & \text{IQL} & \text{DARA} & \text{IGDF} & \text{OTDF} & \text{DT} & \text{QT} & \text{DADT} & \text{TAF-DT} \\
\midrule
half-m   & medium         & 30.0 & 26.6 & 41.6 & 39.1 & 34.6 & 34.5 & 34.8 & \bestcell{44.2}{0.1} \\
half-m   & medium-expert  & 31.8 & 32.0 & 29.6 & 35.6 & 30.8 & -1.3 & 36.5 & \bestcell{42.5}{1.9} \\
half-m   & expert         & 8.5  & 9.3  & 10.0 & 10.7 & 4.7  & 0.8  & 11.5 & \bestcell{69.0}{7.3} \\
half-m-r & medium         & 30.8 & 35.6 & 28.0 & 40.0 & 30.3 & 31.1 & 30.2 & \bestcell{42.9}{2.0} \\
half-m-r & medium-expert  & 12.9 & 16.9 & 12.0 & 34.4 & 19.4 & 24.6 & 25.7 & \bestcell{42.8}{0.6} \\
half-m-r & expert         & 5.9  & 3.7  & 5.3  & 8.2  & 4.7  & 11.3 & 9.5  & \bestcell{53.0}{18.7} \\
half-m-e & medium         & 41.5 & 40.3 & 40.9 & 41.4 & 34.9 & 22.2 & 36.4 & \bestcell{44.4}{0.1} \\
half-m-e & medium-expert  & 25.8 & 30.6 & 26.2 & 35.1 & 36.5 & 20.7 & 37.1 & \bestcell{43.8}{0.5} \\
half-m-e & expert         & 7.8  & 8.3  & 7.5  & 9.8  & 7.7  & 7.6  &  5.4 & \bestcell{73.7}{7.0} \\
\midrule
hopp-m   & medium         & 13.5 & 13.5 & 13.4 & 11.0 & 12.1 & 10.1 & 11.4 & \bestcell{44.7}{16.5} \\
hopp-m   & medium-expert  & 13.4 & 13.6 & 13.3 & 12.6 & 13.2 & 13.2 & 13.1 & \bestcell{36.0}{20.5} \\
hopp-m   & expert         & 13.5 & 13.6 & 13.9 & 10.7  & 12.9 & 13.1 & 13.5 & \bestcell{56.1}{39.1} \\
hopp-m-r & medium         & 10.8 & 10.2 & 12.0 & 8.7  & 13.3 & 13.1 & 14.4 & \bestcell{53.2}{21.5} \\
hopp-m-r & medium-expert  & 11.6  & 10.4  & 8.2 & 9.7 & 12.4 & 15.6 & 12.2 & \bestcell{79.9}{13.0} \\
hopp-m-r & expert         & 9.8 & 9.0 & 11.4 & 10.7 & 12.7 & \bestplain{15.7} & 13.7 & \bestcell{15.7}{2.7} \\
hopp-m-e & medium         & 12.6 & 13.0 & 12.7 & 7.9 & 11.8 & 9.9  & 11.9  & \bestcell{93.5}{4.7} \\
hopp-m-e & medium-expert  & 14.1 & 13.8 & 13.3 & 9.6  & 11.8 & 12.6 & 10.7 & \bestcell{69.7}{27.3} \\
hopp-m-e & expert         & 13.8 & 12.3 & 12.8 & 5.9  & 12.0 & 12.7 & 11.7 & \bestcell{86.5}{21.4} \\
\midrule
walk-m   & medium         & 23.0 & 23.3 & 27.5 & \bestplain{50.5} & 23.7 & 11.5 & 20.8 & 46.6\!\pm\!9.3 \\
walk-m   & medium-expert  & 21.5 & 22.2 & 20.7 & 44.3 & 22.4 & 29.0 & 25.3  & \bestcell{41.1}{5.3} \\
walk-m   & expert         & 20.3 & 17.3 & 15.8 & 55.3 & 15.6 & 23.8 & 28.3 & \bestcell{70.3}{22.1} \\
walk-m-r & medium         & 11.3 & 10.9 & 13.4 & 37.4 & 12.3 & 30.1 & 28.3 & \bestcell{44.8}{5.0} \\
walk-m-r & medium-expert  & 7.0  & 4.5  & 6.9  & 33.8 & 6.0  & 1.6  & 13.6 & \bestcell{40.6}{20.7} \\
walk-m-r & expert         & 6.3  & 4.5  & 5.5  & 41.5 & 10.1 & 1.1  & 9.5  & \bestcell{86.3}{16.1} \\
walk-m-e & medium         & 24.1 & 31.7 & 27.5 & 49.9 & 17.8 & 19.7 & 27.7 & \bestcell{51.4}{9.2} \\
walk-m-e & medium-expert  & 27.0 & 23.3 & 25.3 & \bestplain{40.5} & 14.3 & 24.2 & 25.2 & 28.4\!\pm\!6.3 \\
walk-m-e & expert         & 22.4 & 25.2 & 24.7 & 45.7 & 10.2 & 21.8 & 26.7 & \bestcell{85.5}{9.9} \\
\midrule
ant-m    & medium         & 38.7 & 41.3 & 40.9 & 39.4 & 37.9 & 38.6 & 42.5 & \bestcell{42.6}{0.6} \\
ant-m    & medium-expert  & 47.0 & 43.3 & 44.4 & 58.3 & 48.1 & 1.0  & 44.0 & \bestcell{75.4}{6.2} \\
ant-m    & expert         & 36.2 & 48.5 & 41.4 & 85.4 & 22.8 & -1.0 & 23.7 & \bestcell{85.5}{11.8} \\
ant-m-r  & medium         & 38.2 & 38.9 & 39.7 & 41.2 & 17.5 & 25.0 & 37.8 & \bestcell{41.4}{1.3} \\
ant-m-r  & medium-expert  & 38.1 & 33.4 & 37.3 & 50.8 & 28.6 & 8.2  & 39.0 & \bestcell{78.3}{8.9} \\
ant-m-r  & expert         & 24.1 & 24.5 & 23.6 & 67.2 & 21.2 & 8.3  & 25.9 & \bestcell{75.0}{15.8} \\
ant-m-e  & medium         & 32.9  & 40.2  & 36.1  & 39.9 & 41.3 & 35.1 & 27.4 & \bestcell{42.0}{0.5} \\
ant-m-e  & medium-expert  & 35.7 & 36.5 & 30.7 & 65.7 & 57.3 & 12.8 & 43.1 & \bestcell{69.5}{11.4} \\
ant-m-e  & expert         & 36.1 & 34.6 & 35.2 & \bestplain{86.4} & 37.9 & 12.3 & 31.1 & 81.9\!\pm\!6.9 \\
\midrule
\multicolumn{2}{@{}c|}{Total Score} & 798.0 & 816.8 & 808.7 & 1274.3 & 760.8 & 570.6 & 859.6 & \bestplain{2078.2} \\
\bottomrule
\end{tabular}
\end{table*}

\begin{table*}[ht]
\caption{\textbf{Performance comparison of cross-domain offline RL algorithms under kinematic shifts.} Abbreviations are as in Table~\ref{tab:morph-shift}. We report normalized target-domain performance (\emph{mean $\pm$ std.}) over \textbf{five} seeds; best per row is highlighted.}
\label{tab:kinematic-shift}
\centering
\small
\tighttab
\begin{tabular}{@{}ll 
  | >{$}c<{$}  
  | >{$}c<{$}  
    >{$}c<{$}  
    >{$}c<{$}  
  | >{$}c<{$}  
    >{$}c<{$}  
    >{$}c<{$}  
  | >{$}c<{$}  
  @{}}
\toprule
Source & Target & \text{IQL} & \text{DARA} & \text{IGDF} & \text{OTDF} & \text{DT} & \text{QT} & \text{DADT} & \text{TAF-DT} \\
\midrule
half-m   & medium         & 12.3 & 10.6 & 23.6 & 40.2 & 32.1 & 14.6 & 14.5 & \bestcell{41.2}{0.5} \\
half-m   & medium-expert  & 10.8 & 12.9 & 9.8  & 10.1 & 22.4 & 6.2  & 21.4 & \bestcell{40.8}{1.5} \\
half-m   & expert         & 12.6 & 12.1 & 12.8 & 8.7  & 13.9 & 5.0  & 15.8 & \bestcell{27.5}{5.0} \\
half-m-r & medium         & 10.0 & 11.5 & 11.6 & 37.8 & 11.6 & 10.7 & 8.8  & \bestcell{40.8}{0.3} \\
half-m-r & medium-expert  & 6.5  & 9.2  & 8.6  & 9.7  & 7.5  & 40.1 & 6.0 & \bestcell{41.4}{1.6} \\
half-m-r & expert         & 13.6 & 14.8 & 13.9 & 7.2  & 2.7  & 19.2 & 5.7  & \bestcell{27.6}{7.4} \\
half-m-e & medium         & 21.8 & 25.9 & 21.9 & 30.7 & 17.5 & 18.7 & 14.5 & \bestcell{41.2}{0.9} \\
half-m-e & medium-expert  & 7.6  & 9.5  & 8.9  & 10.9 & 13.1 & 3.7  & 11.4 & \bestcell{35.5}{12.1} \\
half-m-e & expert         & 9.1  & 10.4 & 10.7 & 3.2  & 19.5 & 10.3 & 19.4 & \bestcell{26.0}{14.2} \\
\midrule
hopp-m   & medium         & 58.7 & 43.9 & 65.3 & 65.6 & 16.4 & 19.7 & 3.6  & \bestcell{66.5}{0.9} \\
hopp-m   & medium-expert  & \bestplain{68.5} & 55.4 & 51.1 & 55.4 & 6.3  & 10.9 & 10.4 & 56.2\!\pm\!28.5 \\
hopp-m   & expert         & 79.9 & 83.7 & \bestplain{87.4} & 35.0 & 3.5  & 7.8  & 3.5 & 57.6\!\pm\!32.7 \\
hopp-m-r & medium         & 36.0 & 39.4 & 35.9 & 35.5 & 11.1 & 23.0 & 16.8 & \bestcell{63.1}{3.4} \\
hopp-m-r & medium-expert  & 36.1 & 34.1 & 36.1 & 47.5 & 3.8  & \bestplain{54.0} & 35.3 & 23.7\!\pm\!17.6 \\
hopp-m-r & expert         & 36.0 & 36.1 & 36.1 & 49.9 & 9.8  & 19.9 & 6.7  & \bestcell{62.0}{20.7} \\
hopp-m-e & medium         & 66.0 & 61.1 & 65.2 & 65.3 & 21.6 & 3.4  & 14.3 & \bestcell{66.8}{1.4} \\
hopp-m-e & medium-expert  & 45.1 & 61.9 & \bestplain{62.9} & 38.6 & 10.3 & 16.9 & 6.6  & 49.2\!\pm\!27.3 \\
hopp-m-e & expert         & 44.9 & \bestplain{84.2} & 52.8 & 29.9 & 18.7 & 10.9 & 15.5 & 68.1\!\pm\!16.8 \\
\midrule
walk-m   & medium         & 34.3 & 35.2 & 41.9 & 49.6 & 31.6 & 26.9 & 27.3 & \bestcell{55.7}{11.0} \\
walk-m   & medium-expert  & 30.2 & \bestplain{51.9} & 42.3 & 43.5 & 35.8 & 19.8 & 19.1 & 37.6\!\pm\!8.2 \\
walk-m   & expert         & 56.4 & 40.7 & \bestplain{60.4} & 46.7 & 35.4 & 50.2 & 38.2 & 55.7\!\pm\!8.0 \\
walk-m-r & medium         & 11.5 & 12.5 & 22.2 & 49.7 & 17.9 & 33.7 & 6.8  & \bestcell{54.2}{19.9} \\
walk-m-r & medium-expert  & 9.7  & 11.2 & 7.6  & \bestplain{55.9} & 24.2 & 49.8 & 28.1 & 31.3\!\pm\!9.7 \\
walk-m-r & expert         & 7.7  & 7.4  & 7.5  & 51.9 & 18.4 & 3.1  & 18.0 & \bestcell{53.7}{6.9} \\
walk-m-e & medium         & 41.8 & 38.1 & 41.2 & 44.6 & 38.6 & 5.6  & \bestplain{78.9} & 60.1\!\pm\!4.9 \\
walk-m-e & medium-expert  & 22.2 & 23.6 & 28.1 & 16.5 & 15.2 & 29.2 & 33.0 & \bestcell{51.4}{21.2} \\
walk-m-e & expert         & 26.3 & 36.0 & 46.2 & 42.4 & 39.3 & 25.0 & 32.2 & \bestcell{56.8}{11.5} \\
\midrule
ant-m    & medium         & 50.0 & 42.3 & 54.5 & 55.4 & 31.2 & 22.5 & 17.7 & \bestcell{59.2}{2.0} \\
ant-m    & medium-expert  & 57.8 & 54.1 & 54.5 & \bestplain{60.7} & 13.0 & 7.9  & 13.5 & 60.3\!\pm\!5.4 \\
ant-m    & expert         & 59.6 & 54.2 & 49.4 & \bestplain{90.4} & 7.0  & 7.0  & 11.7 & 88.7\!\pm\!8.9 \\
ant-m-r  & medium         & 43.7 & 42.0 & 41.4 & \bestplain{52.8} & 31.1 & 22.4 & 30.3 & 51.7\!\pm\!4.6 \\
ant-m-r  & medium-expert  & 36.5 & 36.0 & 37.2 & 54.2 & 26.9 & 12.0 & 33.1 & \bestcell{62.8}{1.9} \\
ant-m-r  & expert         & 24.4 & 22.1 & 24.3 & 74.7 & 27.1 & 8.9  & 25.5 & \bestcell{89.9}{5.0} \\
ant-m-e  & medium         & 49.5 & 44.7 & 41.8 & 50.2 & 21.2 & 9.4  & 11.1 & \bestcell{52.2}{4.8} \\
ant-m-e  & medium-expert  & 37.2 & 33.3 & 41.5 & 48.8 & 16.5 & 10.8 & 13.6 & \bestcell{55.6}{3.2} \\
ant-m-e  & expert         & 18.7 & 17.8 & 14.4 & 78.4 & 7.2  & 8.0  & 11.7 & \bestcell{88.5}{10.3} \\
\midrule
\multicolumn{2}{@{}c|}{Total Score} &
  1193.0 & 1219.8 & 1271.0 & 1547.6 & 679.4 & 647.2 & 680 & \bestplain{1900.6} \\
\bottomrule
\end{tabular}
\end{table*}

\begin{table}[ht]
\centering
\caption{Compact summary across the three dynamics-shift families. “Wins” counts the number of tasks (out of 36 per shift type) on which TAF-DT achieves the highest mean normalized score. “2nd-best” denotes the strongest competing method by total score within that shift family.}
\label{tab:shift-summary}
\small
\begin{tabular}{lccc}
\toprule
Shift type & TAF-DT total & 2nd-best total & Wins \\
\midrule
Morphology & \textbf{2078.2} & 1274.3 (OTDF) & 31/36 \\
Kinematic  & \textbf{1900.6} & 1547.6 (OTDF) & 24/36 \\
Gravity    & \textbf{1347.3} & 1160.7 (OTDF) & 19/36 \\
\bottomrule
\end{tabular}
\end{table}

TAF-DT achieves the strongest aggregate performance on all three shift families. Relative to the best competing method by total score, TAF-DT improves the aggregate normalized return by \textbf{63.1\%} on morphology shifts, \textbf{22.8\%} on kinematic shifts, and \textbf{16.1\%} on gravity shifts. The accompanying wins counts show that these gains are not driven by only a few outlier tasks: TAF-DT is best on 31/36 morphology tasks, 24/36 kinematic tasks, and 19/36 gravity tasks. The gains are also not confined to one environment family: TAF-DT wins all \textbf{9} halfcheetah settings under both morphology and kinematic shifts, dominates hopper and ant under gravity shifts, and remains broadly competitive on rows where another method is locally better. Taken together, these results support aggregate robustness under dynamics shift rather than only a few isolated large wins.

Tables \ref{tab:morph-shift} and \ref{tab:kinematic-shift} further show that TAF-DT consistently outperforms sequence-model baselines (DT, QT, DADT) and usually exceeds recent cross-domain offline RL baselines such as DARA, IGDF, and OTDF. Under morphology shifts, TAF-DT achieves the highest mean score on \textbf{31/36} tasks; under kinematic shifts, it does so on \textbf{24/36} tasks. In the remaining cases, TAF-DT is typically close to the best row-wise method, which suggests that the benefit comes from broad improvement in sequence stitching quality rather than a few isolated large wins. Aggregate return alone, however, does not reveal whether the learned sequences have become more stable exactly where source and target fragments are stitched together. We therefore next turn from performance comparison to sequence-level mechanism validation at stitch junctions.

\begingroup
\setlength{\abovecaptionskip}{2pt}
\setlength{\belowcaptionskip}{0pt}
\begin{figure*}[ht]
    \centering
    \subfloat[]{\includegraphics[width=0.31\textwidth]{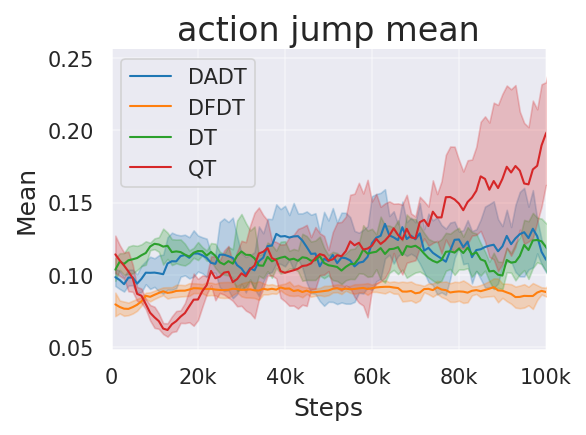}}\hfill
    \subfloat[]{\includegraphics[width=0.31\textwidth]{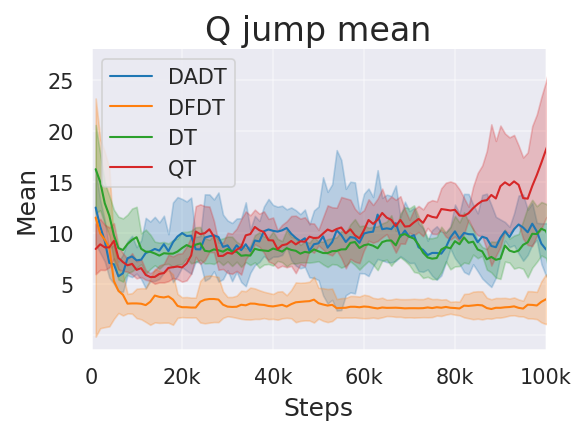}}\hfill
    \subfloat[]{\includegraphics[width=0.31\textwidth]{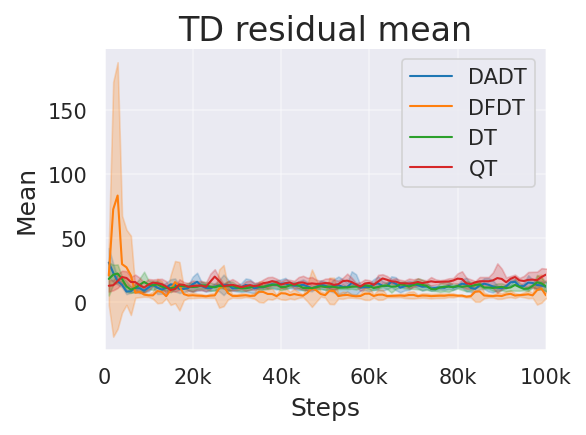}}
    \caption{Sequence-level diagnostics at stitch junctions during training: mean action jump, $Q$-value jump, and local TD residual. Lower is better. The figure serves as mechanism validation that stitched sequences can create localized discontinuities and that TAF-DT yields the smoothest and most stable token transitions among the compared sequence models.}
    \label{fig:jumps}
\end{figure*}
\endgroup

\subsection{Sequence-level diagnostics of stitch-junction stability}

We next test the mechanism claim directly: stitched source--target sequences can introduce localized discontinuities at junctions, and the target-aligned fusion rule should suppress them.

To directly probe sequence semantics at stitch junctions, we precompute the junction index set $\mathcal{J}$ for every relabeled training sequence, namely any boundary where two fragments are concatenated or where the sequence crosses between retained source fragments and target-domain context. At each training checkpoint, we evaluate three quantities on a fixed validation pool of such sequences: (i) the action jump $J_a=\mathbb{E}_{t^\star\in\mathcal{J}}\|\pi(s_{t^\star})-\pi(s_{t^\star-1})\|_2$; (ii) the $Q$-jump $J_Q=\mathbb{E}_{t^\star\in\mathcal{J}}|Q(s_{t^\star},\pi(s_{t^\star}))-Q(s_{t^\star-1},\pi(s_{t^\star-1}))|$; and (iii) the local TD residual $\mathbb{E}_{t\in\mathcal{N}(t^\star)}|r_t+\gamma V(s_{t+1})-V(s_t)|$, where $\mathcal{N}(t^\star)=\{t:|t-t^\star|\le w\}\cap\{1,\dots,T-1\}$ with $w=2$. Figure~\ref{fig:jumps} reports moving means over checkpoints for TAF-DT, DADT, QT, and DT under the same backbone, budget, and data.

The three diagnostics close the loop on the sequence-level claim motivating TAF-DT. First, they show that stitch junctions are indeed a concrete source of localized instability under weaker fusion rules: DT and DADT maintain visibly larger action jumps, while QT drifts upward to roughly $0.25$--$0.30$, indicating that naive or weakly controlled stitching leaves substantial token discontinuity near fragment boundaries. Second, they show that the target-aligned fusion rule suppresses this instability in a consistent and sustained way. TAF-DT keeps the action-jump mean near $0.06$--$0.09$ throughout training, drives $Q$-jumps rapidly down to a low and stable range around $2$--$4$, and, after a brief early spike, settles to the lowest sustained TD residual level, around $3$--$6$, while the competing models remain higher and more volatile for most of training. Taken together, the evidence supports the intended mechanism rather than only the final score: reliable source fusion should be judged by stitch-junction stability at the sequence level, and TAF-DT improves exactly that property through better fragment selection, more stable post-stitch sequence semantics, and smoother value behavior around the junction.

\subsection{Ablation studies}
\label{sec:main_ablation}

We now isolate the three components most directly implied by the target-Bellman-risk-matching view: fragment filtering, advantage-conditioned tokens, and $Q$-regularized sequence training.

\begin{table*}[htb]
\centering
\caption{Ablation on fragment filtering.}
\label{tab:data_filter_ablation}
\small
\begin{tabular}{lll|ccc}
\toprule
Source & Target & Shift & TAF-DT & OT-only & MMD-only \\
\midrule
ant-m & medium & gravity & $\textbf{63.8}\pm1.4$ & $59.6\pm6.1$ & $62.6\pm2.4$ \\
hopp-m & medium & kinematic & $\textbf{66.5}\pm0.9$ & $17.3\pm8.8$ & $64.2\pm1.5$ \\
hopp-m-r & medium & morph & $\textbf{56.7}\pm0.0$ & $50.0\pm1.1$ & $49.0\pm2.7$ \\
walk-m  & medium & morph-expert & $\textbf{41.1}\pm5.3$ & $31.7\pm1.1$ & $40.4\pm0.0$ \\
\bottomrule
\end{tabular}
\end{table*}

\subsubsection{Ablation on fragment filtering}

Table~\ref{tab:data_filter_ablation} compares full TAF-DT (MMD+OT) with MMD-only and OT-only variants on four representative source--target pairs spanning gravity, kinematic, and morphology shifts. TAF-DT achieves the highest return in all four settings, which supports the gate-then-weight principle implied by Theorem~\ref{thm:bellman-risk}. The largest gap appears on the kinematic task \texttt{hopp-m}$\to$\texttt{medium}, where OT-only collapses to $17.3\pm 8.8$ while TAF-DT reaches $66.5\pm 0.9$, suggesting that action-feasibility weighting alone is insufficient when the retained source support is not first cleaned by a state-structure gate. MMD-only is substantially stronger than OT-only and often close to TAF-DT, confirming that state-structure alignment is the first-order ingredient. However, the full two-level design remains consistently best, especially on harder morphology and gravity settings where the additional OT weighting improves action-level feasibility after MMD screening.

\begin{table}[htb]
\centering
\caption{Ablation on advantage-conditioned tokens under the kinematic shift.}
\label{tab:adv_ablation}
\small
\begin{tabular}{lll|cc}
\toprule
Source   & Target         & Advantage & RTG \\
\midrule
ant-m    & medium         & $\textbf{59.2}\pm2.0$ & $52.2\pm2.7$ \\
hopp-m   & medium         & $\textbf{66.5}\pm0.9$ & $64.6\pm2.3$ \\
hopp-m   & medium-expert  & $\textbf{78.5}\pm2.6$ & $49.7\pm24.8$ \\
\bottomrule
\end{tabular}
\end{table}

\subsubsection{Ablation on advantage-conditioned tokens}
Table~\ref{tab:adv_ablation} shows that replacing return-to-go tokens with advantage-conditioned tokens consistently improves TAF-DT on representative kinematic-shift tasks. The gains are modest but stable on \texttt{ant-m}$\to$\texttt{medium} ($59.2$ vs.\ $52.2$) and \texttt{hopp-m}$\to$\texttt{medium} ($66.5$ vs.\ $64.6$), and become large on the more demanding \texttt{hopp-m}$\to$\texttt{medium-expert} setting ($78.5$ vs.\ $49.7$). This pattern is consistent with the sequence-aware stitching mechanism in Section~\ref{sec:adv-tokens}: advantage tokens provide a more local and scale-robust supervision signal than absolute RTG values when trajectories of heterogeneous quality are mixed across domains.

\subsubsection{Ablation on Q-regularization coefficient}

Table~\ref{tab:q_ablation} studies the coefficient $\alpha$ of the junction-aware $Q$-regularizer. The main trend is that a nonzero regularizer is often beneficial on difficult cross-domain tasks, but the optimal strength is task dependent. For example, on \texttt{ant-m}$\to$\texttt{medium} (kinematic), performance improves from $50.5\pm4.5$ at $\alpha=0$ to $59.2\pm2.0$ at $\alpha=3.5$; on \texttt{walk-m-e}$\to$\texttt{medium} (kinematic), it rises from $46.2\pm12.5$ to $59.6\pm7.8$ at $\alpha=5.0$. At the same time, some tasks such as \texttt{hopp-m}$\to$\texttt{medium} (gravity) perform best with no regularization, while others are nearly insensitive. We therefore use $\alpha=3.5$ as a robust default: it is best or near-best on several representative tasks and avoids overfitting the method to a single shift family.

\begin{table*}[ht]
\centering
\caption{Ablation on the $Q$-regularization coefficient $\alpha$. Results show the mean and standard deviation over the last three checkpoints from a single seed.}
\label{tab:q_ablation}
\small
\begin{tabular}{lll|cccc}
\toprule
Source  & Target & Shift & $\alpha=0.0$ & $\alpha=1.0$ & $\alpha=3.5$ & $\alpha=5.0$ \\
\midrule
half-m & medium & morph & $43.9\pm0.4$ & $\textbf{44.2}\pm0.1$ &  $43.3\pm0.2$ & $43.9\pm0.1$ \\
half-m-e & medium & morph & $\textbf{44.4}\pm0.1$ & $44.2\pm0.4$ & $44.2\pm0.6$ & $44.1\pm0.2$ \\
hopp-m-e & medium & morph & $25.1\pm1.0$ & $27.1\pm2.2$ &  $56.5\pm25.7$ &  $\textbf{93.5}\pm4.7$ \\
hopp-m-e & medium & kinematic & $62.3\pm2.4$ & $\textbf{65.8}\pm1.9$ & $63.1\pm2.2$ & $63.4\pm0.9$ \\
hopp-m & medium & gravity  & $\textbf{82.4}\pm7.5$ & $59.8\pm22.3$ & $60\pm22.3$ & $42.3\pm7.3$ \\
hopp-m-e & medium & gravity & $61.7\pm32.3$ & $62.3\pm17.0$ & $61.2\pm27.3$ & $\textbf{63.6}\pm18.0$ \\
walk-m-e & medium & morph & $42.3\pm8.3$ & $41.2\pm13.1$ &  $\textbf{51.4}\pm9.2$ & $45.6\pm3.0$ \\
walk-m-e & medium & kinematic & $46.2\pm12.5$ & $50.4\pm14.3$ & $47.4\pm21.9$ & $\textbf{59.6}\pm7.8$ \\
ant-m & medium & kinematic & $50.5\pm4.5$ & $56.6\pm3.3$ & $\textbf{59.2}\pm2.0$ & $53.2\pm7.7$ \\
ant-m-e & medium & gravity & $62.4\pm1.1$ & $65.1\pm5.1$ & $63.8\pm9.5$ & $\textbf{68.9}\pm6.4$ \\
\bottomrule
\end{tabular}
\end{table*}

\section{Discussion and Limitations}\label{sec:discussion}

This work should be read first as a formulation and Bellman-risk-level analysis of \emph{target-aligned sequence fusion}, and only second as one particular Decision-Transformer-based instantiation. We intentionally keep the formal guarantee at the target Bellman-risk level, because propagating such control to global policy performance under shifted offline sequence learning would require stronger concentrability, coverage, and policy-proximity assumptions that are difficult to verify in practice. The empirical study therefore tests whether the risk-matching fusion rule improves both target-domain return and stitch-junction stability rather than treating the theory as a full return guarantee.

Our empirical evaluation is deliberately centered on simulated D4RL-style MuJoCo benchmarks with controlled dynamics shifts, which provide a concrete testbed for studying how imported source trajectories affect target-side critic reliability, sequence stitchability, and downstream performance. This choice enables repeatable comparisons under well-specified shift families, but it also means that the present evidence remains concentrated in a sequential-control setting rather than in broader real-world transfer pipelines. In addition, the current study instantiates the proposed target-aligned fusion principle with a Decision Transformer backbone so that the interaction between fusion, sequence construction, and training can be analyzed cleanly. The formulation-level claim concerns the target-risk-driven gate--then--weight principle; by contrast, the precise MMD features, OT cost geometry, and critic-conditioning hyperparameters used in TAF-DT are implementation choices whose best setting can remain task- and shift-dependent. The same feasibility-weighted fusion distribution $\mathbb P_{\mathrm{mix}}^{I,w}$ is, in principle, compatible with TD-based offline RL methods such as IQL, with actor--critic pipelines, and more broadly with other decision-sequence learners that must absorb externally sourced trajectories in broader distribution-shifted settings. Extending the present framework to larger-scale settings, broader embodiment shifts, real-world transfer scenarios, and non-transformer sequence learners is therefore an important direction for future work.

\section{Conclusion}\label{sec:conclusion}

Motivated by distribution-shifted decision-sequence learning, we studied how a target-domain sequence learner with Decision Transformers should absorb externally sourced trajectories, and developed that question concretely in a dynamics-shifted offline control setting. The central message is that imported source trajectories should not be fused heuristically; they should be admitted only insofar as the resulting fused distribution preserves target-side Bellman reliability. In our setting, this leads to a gate--then--weight fusion rule, followed by sequence-aware stitching and lightweight target-aligned Transformer training. Our analysis ties target Bellman risk to measurable state-structure and transport mismatches, clarifying why MMD-based selection and OT-based weighting play complementary roles. Across gravity, kinematic, and morphology shifts, TAF-DT achieves the strongest aggregate returns together with more stable decision-sequence semantics. The present evidence supports target-aligned risk matching as a principled alternative to heuristic data preprocessing in shifted sequential learning.

\bibliographystyle{IEEEtran}
\bibliography{iclr2026_conference}
\clearpage
\appendices
\twocolumn[
\begin{center}
{\LARGE\bfseries Supplementary Material\par}
\vspace{0.8em}
\end{center}
]

\section{Additional Lemma~\ref{lem:exp} and its proof}
In the appendix we only use two standard facts. First, for a measurable map $T$ and a distribution $\mu$, we write $T_{\#}\mu$ for the pushforward distribution of $T(x)$ when $x\sim\mu$. Second, we use the Kantorovich--Rubinstein duality, which gives the usual way to control expectation differences by the $1$-Wasserstein distance.

\begin{definition}[Kantorovich--Rubinstein duality]
Let $(\mathcal X,d)$ be a metric space and let $\mu,\nu$ be probability measures on $\mathcal X$ with finite first moments. The $1$-Wasserstein distance is
\[
W_1(\mu,\nu)
:= \inf_{\pi\in\Pi(\mu,\nu)} \int_{\mathcal X\times\mathcal X} d(x,y)\, \mathrm d\pi(x,y),
\]
where $\Pi(\mu,\nu)$ is the set of couplings of $\mu$ and $\nu$. Its dual form is
\[
W_1(\mu,\nu)
= \sup_{\|f\|_{\mathrm{Lip}}\le 1}
\Big\{\int_{\mathcal X} f\,\mathrm d\mu - \int_{\mathcal X} f\,\mathrm d\nu\Big\}.
\]
Consequently, any $L$-Lipschitz function $g$ satisfies
\[
\Big| \mathbb E_{\mu}[g] - \mathbb E_{\nu}[g] \Big|
\le L\,W_1(\mu,\nu).
\]
\end{definition}

\begin{restatable}[Expectation deviation under the weighted data fusion]{lemma}{expdev}\label{lem:exp}
For any $1$-Lipschitz $g(u)$ and any $h$ with $\|h\|_{\mathcal H}\le 1$, we have 
$
\big|\mathbb E_{\mathbb P_{\mathrm{mix}}^{I,w}}g-\mathbb E_{\mathbb P_T}g\big|
\le \beta\,\Delta_w$ and 
$
\big|\mathbb E_{\mathbb P_{\mathrm{mix}}^{I,w}}h-\mathbb E_{\mathbb P_T}h\big|
\le \beta\,\Delta_m.
$
\end{restatable}

\begin{proof}
Let $\pi_z:\mathcal U\to\mathcal Z$ map $u=(s,a,s')$ to the latent state $z=f_\phi(s)$. Denote the pushforward marginals by
$\mu_T:=\pi_{z\#}\mathbb P_T$ and $\mu_S^{\,w}:=\pi_{z\#}\mathbb P_S^{I,w}$.
For each retained fragment $\tau^S$ (i.e., $I(\tau^S)=1$), let $\mu_\tau$ be its latent-state (empirical or normalized) distribution.

\smallskip
\textbf{Step 1 (Lipschitz part via Kantorovich--Rubinstein duality).}
By linearity of expectation under the convex mixture, we have
\begin{align*}
&\mathbb E_{\mathbb P_{\mathrm{mix}}^{I,w}}g
=(1-\beta)\,\mathbb E_{\mathbb P_T}g+\beta\,\mathbb E_{\mathbb P_S^{I,w}}g \\
&\Rightarrow\
\mathbb E_{\mathbb P_{\mathrm{mix}}^{I,w}}g-\mathbb E_{\mathbb P_T}g
= \beta\bigl(\mathbb E_{\mathbb P_S^{I,w}}g-\mathbb E_{\mathbb P_T}g\bigr).
\end{align*}
Taking absolute values and applying the Kantorovich--Rubinstein duality on $(\mathcal U,\rho)$,
\[
\sup_{\mathrm{Lip}(g)\le 1}\bigl|\mathbb E_{\mathbb P_S^{I,w}}g-\mathbb E_{\mathbb P_T}g\bigr|
= W_1(\mathbb P_S^{I,w},\mathbb P_T)=\Delta_w.
\]
Hence, for any $1$-Lipschitz $g$,
\[
\bigl|\mathbb E_{\mathbb P_{\mathrm{mix}}^{I,w}}g-\mathbb E_{\mathbb P_T}g\bigr|
\le \beta\,\Delta_w.
\]

\smallskip
\textbf{Step 2 (MMD part on latent states).}
For $h\in\mathcal H$ acting on $z$, expectations under triple distributions reduce to those under their latent pushforwards:
$\mathbb E_{\mathbb P}h:=\mathbb E_{z\sim \pi_{z\#}\mathbb P}[h(z)]$.
As above,
\[
\mathbb E_{\mathbb P_{\mathrm{mix}}^{I,w}}h-\mathbb E_{\mathbb P_T}h
= \beta\bigl(\mathbb E_{\mathbb P_S^{I,w}}h-\mathbb E_{\mathbb P_T}h\bigr)
= \beta\bigl(\mathbb E_{\mu_S^{\,w}}h-\mathbb E_{\mu_T}h\bigr).
\]
Taking the supremum over the unit RKHS ball and using the kernel mean embedding characterisation of MMD,
\[
\sup_{\|h\|_{\mathcal H}\le 1}\bigl|\mathbb E_{\mu_S^{\,w}}h-\mathbb E_{\mu_T}h\bigr|
=\mathrm{MMD}_k(\mu_S^{\,w},\mu_T).
\]
Therefore, for any $\|h\|_{\mathcal H}\le 1$,
\[
\bigl|\mathbb E_{\mathbb P_{\mathrm{mix}}^{I,w}}h-\mathbb E_{\mathbb P_T}h\bigr|
\le \beta\,\mathrm{MMD}_k(\mu_S^{\,w},\mu_T).
\]

\smallskip
\textbf{Step 3 (Bounding $\mathrm{MMD}_k(\mu_S^{\,w},\mu_T)$ by $\Delta_m$).}
Since $\mu_S^{\,w}$ is the latent marginal induced by the retained-fragment mixture, it can be written as
\[
\mu_S^{\,w}=\sum_{\tau^S:I(\tau^S)=1}\alpha_\tau\,\mu_\tau,
\qquad
\alpha_\tau\ge 0,
\quad
\sum_\tau \alpha_\tau=1.
\]
Let $m_\mu:=\mathbb E_{z\sim\mu}[k(z,\cdot)]$ denote the RKHS mean embedding. Then
\[
m_{\mu_S^{\,w}}=\sum_\tau \alpha_\tau\,m_{\mu_\tau}.
\]
Hence, by the triangle inequality,
\begin{align*}
\mathrm{MMD}_k(\mu_S^{\,w},\mu_T)
&=\bigl\|m_{\mu_S^{\,w}}-m_{\mu_T}\bigr\|_{\mathcal H_k} \\
&=\Bigl\|\sum_\tau \alpha_\tau\bigl(m_{\mu_\tau}-m_{\mu_T}\bigr)\Bigr\|_{\mathcal H_k} \\
&\le \sum_\tau \alpha_\tau\bigl\|m_{\mu_\tau}-m_{\mu_T}\bigr\|_{\mathcal H_k} \\
&= \sum_\tau \alpha_\tau\,\mathrm{MMD}_k(\mu_\tau,\mu_T) \\
&\le \sup_{\tau:I_m(\tau)=1}\mathrm{MMD}_k(\mu_\tau,\mu_T).
\end{align*}
Now let $\mu_{\tau^T}$ denote the latent marginal induced by a target fragment $\tau^T$, so that
\[
\mu_T=\mathbb E_{\tau^T\sim\mathbb P_T}[\mu_{\tau^T}].
\]
By convexity of MMD in its second argument,
\begin{align*}
\mathrm{MMD}_k(\mu_\tau,\mu_T)
&=\mathrm{MMD}_k\!\Bigl(\mu_\tau,\mathbb E_{\tau^T\sim\mathbb P_T}[\mu_{\tau^T}]\Bigr)\\
&\le
\mathbb E_{\tau^T\sim\mathbb P_T}\bigl[\mathrm{MMD}_k(\mu_\tau,\mu_{\tau^T})\bigr]
=d^m(\tau).
\end{align*}
Therefore,
\[
\sup_{\tau:I(\tau)=1}\mathrm{MMD}_k(\mu_\tau,\mu_T)
\le
\sup_{\tau:I(\tau)=1} d^m(\tau)
=\Delta_m.
\]
Combining the above inequalities yields $\mathrm{MMD}_k(\mu_S^{\,w},\mu_T)\le \Delta_m$, and therefore
\[
\bigl|\mathbb E_{\mathbb P_{\mathrm{mix}}^{I,w}}h-\mathbb E_{\mathbb P_T}h\bigr|
\le \beta\,\Delta_m,\qquad \forall\,\|h\|_{\mathcal H}\le 1.
\]

Combining the Lipschitz/Wasserstein bound (Step~1) and the RKHS/MMD bound (Steps~2--3) yields
\[
\bigl|\mathbb E_{\mathbb P_{\mathrm{mix}}^{I,w}}g-\mathbb E_{\mathbb P_T}g\bigr|
\le \beta\,\Delta_w,\quad
\bigl|\mathbb E_{\mathbb P_{\mathrm{mix}}^{I,w}}h-\mathbb E_{\mathbb P_T}h\bigr|
\le \beta\,\Delta_m,
\]
as claimed.
\end{proof}

\section{Additional lemma on latent value heads and its proof}\label{sec:proof-of-lem-valapprox}
This appendix only needs a lightweight representation argument. In practice, the encoder produces a latent code $z=f_\phi(s)$ and the value head is computed from that code. The desired approximation chain is
\[
V(s)\ \approx\ \widetilde V(f_\phi(s))\ \approx\ h_V(f_\phi(s)).
\]
The first approximation says that the encoder is already approximately value-sufficient: once $s$ is mapped to $z=f_\phi(s)$, the remaining value ambiguity is at most $\varepsilon_H$. The second approximation explains why the value head on the latent code may be taken from an RKHS function class. This is where universality of the kernel is used.

A continuous kernel $k$ on a compact metric space $K$ is called \emph{universal} if its RKHS $\mathcal H_k$ is dense in $C(K)$ under the uniform norm, i.e., for every continuous $g:K\to\mathbb R$ and every $\eta>0$, there exists $h\in\mathcal H_k$ such that
\[
\sup_{z\in K}|g(z)-h(z)|\le \eta.
\]
This property means that the RKHS is expressive enough to uniformly approximate any continuous latent prediction rule on $K$. Standard examples include the Gaussian RBF kernel and the Laplace kernel on compact subsets of $\mathbb R^d$.

\begin{restatable}[Latent value-head representation under a universal kernel]{lemma}{approxvalue}\label{lem:value-approx-from-ass}
Let $K=\overline{f_\phi(\mathcal S)}\subset\mathbb R^d$ be the latent image of the encoder, and let $k$ be a universal kernel on the compact metric space $K$ with RKHS $\mathcal H$. Assume that there exists a continuous latent function $\widetilde V:K\to\mathbb R$ such that
\[
\sup_{s\in\mathcal S}\,|V(s)-\widetilde V(f_\phi(s))|\le \varepsilon_H.
\]
Then for every $\eta>0$ there exists $h_V\in\mathcal H$ such that
\[
\sup_{s\in\mathcal S}\,|V(s)-h_V(f_\phi(s))|\le \varepsilon_H+\eta.
\]
Equivalently, once the encoder preserves value information up to $\varepsilon_H$, restricting the value head to the RKHS incurs only an additional uniform approximation error $\eta$.
\end{restatable}

\begin{proof}
The argument is a direct two-step approximation.

\emph{Step 1: latent representation error.}
By assumption, the value function is already approximately a function of the latent code in the sense that
\[
\sup_{s\in\mathcal S}|V(s)-\widetilde V(f_\phi(s))|\le \varepsilon_H.
\]
This is the representation error induced by the encoder.

\emph{Step 2: latent value-head approximation on the latent space.}
Because $k$ is universal on the compact metric space $K$, the RKHS $\mathcal H$ is dense in $C(K)$ under the uniform norm. Since $\widetilde V\in C(K)$, for any prescribed $\eta>0$ there exists $h_V\in\mathcal H$ such that
\[
\sup_{z\in K}|\widetilde V(z)-h_V(z)|\le \eta.
\]
In other words, the continuous latent value map $\widetilde V$ can be approximated uniformly by a latent value head $h_V$ drawn from the RKHS.

Combining the two steps, for every $s\in\mathcal S$ we have
\[
\begin{aligned}
&\quad |V(s)-h_V(f_\phi(s))|\\
&\le |V(s)-\widetilde V(f_\phi(s))|
   +|\widetilde V(f_\phi(s))-h_V(f_\phi(s))| \\
&\le \varepsilon_H + \sup_{z\in K}|\widetilde V(z)-h_V(z)| \\
&\le \varepsilon_H+\eta.
\end{aligned}
\]
Taking the supremum over $s$ proves the claim.
\end{proof}

\begin{remark}
The latent function $\widetilde V$ is the intuitive object one usually has in mind when writing a value head on top of an encoder: first compute $z=f_\phi(s)$, then predict the value from $z$. The lemma simply formalizes that intuition. The assumption
\[
V(s)\approx \widetilde V(f_\phi(s))
\]
means that the encoder has already retained the value-relevant information, while universality of the kernel ensures that the chosen function class for the head is not overly restrictive. Thus the final predictor can be written in the familiar form $h_V(f_\phi(s))$ with $h_V\in\mathcal H$.
\end{remark}

\section{Additional Lemma~\ref{lem:bellman} and its proof}
\begin{restatable}[Weighted Bellman error transfer]{lemma}{errortrans}\label{lem:bellman}
Let $y=r(s,a)+\gamma V(s')$. There exist constants $R_1,R_2>0$ (depending on $R_{\max}$, $L_r$, and the encoder/kernel bounds) such that
\begin{equation}
\begin{aligned}
&\big|\mathbb E_{\mathbb P_{\mathrm{mix}}^{I,w}}[y-V(s)]-\mathbb E_{\mathbb P_T}[y-V(s)]\big|\\
&\qquad\qquad\qquad\;\le\;\beta\,(R_1\,\Delta_m+R_2\,\Delta_w)\ +\ 4\beta\,\varepsilon_H.
\end{aligned}
\end{equation}
\end{restatable}
\begin{proof}
\textbf{Step 1 (Bounding the $V$-residual transfer).}
Write the one-step TD residual as
\[
R(u)\;=\;y-V(s)\;=\;r(s,a)+\gamma\,V(s')-V(s),
\qquad u=(s,a,s').
\]
By linearity of expectation under the mixture,
\[
\mathbb E_{\mathbb P_{\mathrm{mix}}^{I,w}}R-\mathbb E_{\mathbb P_T}R
= \beta\big(\mathbb E_{\mathbb P_{S}^{I,w}}R-\mathbb E_{\mathbb P_T}R\big).
\]
Hence
\begin{equation}
\begin{aligned}
\big\lvert \mathbb E_{\mathbb P_{\mathrm{mix}}^{I,w}}R-\mathbb E_{\mathbb P_T}R \big\rvert
\le &\beta\Big(
\underbrace{\big\lvert \mathbb E_{\mathbb P_{S}^{I,w}}r(s,a)-\mathbb E_{\mathbb P_T}r(s,a)\big\rvert}_{\text{(I)}}\\
&\ +\underbrace{\big\lvert \mathbb E_{\mathbb P_{S}^{I,w}}V(s)-\mathbb E_{\mathbb P_T}V(s)\big\rvert}_{\text{(II)}}\\
&\ +\gamma\underbrace{\big\lvert \mathbb E_{\mathbb P_{S}^{I,w}}V(s')-\mathbb E_{\mathbb P_T}V(s')\big\rvert}_{\text{(III)}}
\Big).
\end{aligned}
\label{eq:diff_res}
\end{equation}

\emph{Term (I): reward transfer via $W_1$.}
Under Assumption~\ref{ass:bounded}, the reward function is Lipschitz on the transition space $(\mathcal U,\rho)$, so by the Kantorovich--Rubinstein duality,
\[
\big\lvert \mathbb E_{\mathbb P_{S}^{I_m,w}}r(s,a)-\mathbb E_{\mathbb P_T}r(s,a)\big\rvert
\le L_r\,W_1(\mathbb P_{S}^{I,w},\mathbb P_T)
= L_r\,\Delta_w.
\]
Absorb $L_r$ into a constant $R_2>0$ to obtain
\begin{equation}
\big\lvert \mathbb E_{\mathbb P_{S}^{I,w}}r(s,a)-\mathbb E_{\mathbb P_T}r(s,a)\big\rvert
\le R_2\,\Delta_w.
\label{eq:diff_r}
\end{equation}

\emph{Term (II): current-state value transfer via Lemmas~\ref{lem:exp} and \ref{lem:value-approx-from-ass}.}
Let $z=f_\phi(s)$ and denote the corresponding latent pushforwards by
$\mu_T:=\pi_{z\#}\mathbb P_T$ and $\mu_S^{\,w}:=\pi_{z\#}\mathbb P_S^{I,w}$.
By Lemma~\ref{lem:value-approx-from-ass}, there exists a latent value head $h_V\in\mathcal H$ with
\[
\sup_{s\in\mathcal S}|V(s)-h_V(f_\phi(s))|\le \varepsilon_H+\eta.
\]
Therefore,
\begin{align*}
&\quad\big\lvert \mathbb E_{\mathbb P_{S}^{I,w}}V(s)-\mathbb E_{\mathbb P_T}V(s)\big\rvert \\
&\le \big\lvert \mathbb E_{z\sim \mu_S^{\,w}} h_V(z)-\mathbb E_{z\sim \mu_T} h_V(z)\big\rvert + 2(\varepsilon_H+\eta) \\
&\le \|h_V\|_{\mathcal H}\,\mathrm{MMD}_k(\mu_S^{\,w},\mu_T)+2(\varepsilon_H+\eta) \\
&\le C_V\,\mathrm{MMD}_k(\mu_S^{\,w},\mu_T)+2(\varepsilon_H+\eta),
\end{align*}
where $C_V:=\|h_V\|_{\mathcal H}$ is absorbed into constants. By Step~3 of Lemma~\ref{lem:exp},
\(\mathrm{MMD}_k(\mu_S^{\,w},\mu_T)\le \Delta_m\), hence
\begin{equation}
\big\lvert \mathbb E_{\mathbb P_{S}^{I,w}}V(s)-\mathbb E_{\mathbb P_T}V(s)\big\rvert
\le C_V\,\Delta_m+2(\varepsilon_H+\eta).
\label{eq:diff_V}
\end{equation}

\emph{Term (III): next-state value transfer via the same latent argument.}
Let $z'=f_\phi(s')$ and denote the next-state latent pushforwards by
$\mu_T':=\pi_{z'\#}\mathbb P_T$ and $\mu_S^{\,w,'}:=\pi_{z'\#}\mathbb P_S^{I,w}$.
Applying Lemma~\ref{lem:value-approx-from-ass} to the next-state variable $s'\in\mathcal S$ and repeating the same argument gives
\begin{align*}
&\quad\big\lvert \mathbb E_{\mathbb P_{S}^{I,w}}V(s')-\mathbb E_{\mathbb P_T}V(s')\big\rvert \\
&\le \big\lvert \mathbb E_{z'\sim \mu_S^{\,w,'}} h_V(z')-\mathbb E_{z'\sim \mu_T'} h_V(z')\big\rvert + 2(\varepsilon_H+\eta) \\
&\le C_V\,\mathrm{MMD}_k(\mu_S^{\,w,'},\mu_T')+2(\varepsilon_H+\eta).
\end{align*}
Moreover, the same convexity and pushforward argument as in Lemma~\ref{lem:exp} applies to the next-state projection $\pi_{z'}(u)=f_\phi(s')$, so $\mathrm{MMD}_k(\mu_S^{\,w,'},\mu_T')\le \Delta_m$. Hence
\begin{equation}
\big\lvert \mathbb E_{\mathbb P_{S}^{I,w}}V(s')-\mathbb E_{\mathbb P_T}V(s')\big\rvert
\le C_V\,\Delta_m+2(\varepsilon_H+\eta).
\label{eq:diff_Vnext}
\end{equation}

Combining \eqref{eq:diff_r}, \eqref{eq:diff_V}, and \eqref{eq:diff_Vnext} in \eqref{eq:diff_res}, and using $\gamma\le 1$, we obtain
\begin{align*}
\big\lvert \mathbb E_{\mathbb P_{\mathrm{mix}}^{I,w}}R-\mathbb E_{\mathbb P_T}R \big\rvert
\le \beta\Big(R_2\Delta_w + 2C_V\Delta_m + 4(\varepsilon_H+\eta)\Big).
\end{align*}
Let $R_1:=2C_V$ and send $\eta\downarrow 0$. This proves the stated inequality with an explicit $4\beta\varepsilon_H$ fiber-error term.
\end{proof}

\begin{corollary}[Weighted Bellman residual transfer for $\pi_{\rm mix}$]\label{cor:transfer-pimix}
Let 
\[
R(s,a,s'):=r(s,a)+\gamma V(s')-V(s)
\]
be the one-step TD residual as in Lemma~\ref{lem:bellman}. For any measurable $\varphi:\mathcal S\to\mathbb R$ with $\|\varphi\|_\infty\le 1$, there exist constants $R_1,R_2>0$ (reusing the same symbols as in Lemma~\ref{lem:bellman}, up to constant absorption) such that
\begin{align*}
&\Big|\ \mathbb E_{\mathbb P_T}\!\big[\varphi(s)\,R(s,a,s')\big]\ -\
\mathbb E_{\mathbb P_{\mathrm{mix}}^{I,w}}\!\big[\varphi(s)\,R(s,a,s')\big]\ \Big|\\
&\qquad\qquad\qquad\qquad\quad \le\ \beta\,(R_1\Delta_m+R_2\Delta_w)\ +\ 4\beta\,\varepsilon_H.
\end{align*}
\emph{Proof sketch.}
Repeat the proof of Lemma~\ref{lem:bellman} with the bounded multiplier $\varphi(s)$ inserted throughout and with the target-side law replaced by $\mathbb P_T$. Since $\|\varphi\|_\infty\le 1$, the reward term is still controlled by the same $W_1$ argument up to constant absorption, while the two value terms are handled by the same latent value-head bounds as in the proof of Lemma~\ref{lem:bellman}. This yields the stated residual-form inequality. \qed
\end{corollary}

\section{Proof of Theorem~\ref{thm:bellman-risk}}\label{sec:proof_of_target_risk}
\begin{proof}
Define the target-domain Bellman residual by Eq.~\eqref{eq:def-deltaV-risk}, and let $\mu_T^{s}:=\operatorname{Marg}_{s}(\mathbb P_T)$. Since $\delta_T^V$ is state-measurable and $\ell_V(u)=|u|$, Eq.~\eqref{eq:bellman-risk-def} gives
\[
\mathfrak R_{\mathrm{Bell}}({I,w})
=
\|\delta_T^V\|_{1,\mu_T^{s}}.
\]
For the value residual,
\begin{align*}
\|\delta_T^V\|_{1,\mu_T^{s}}
&\le
\|\delta_{\mathrm{mix}}^V\|_{1,\mu_T^{s}}
+
\|\delta_T^V-\delta_{\mathrm{mix}}^V\|_{1,\mu_T^{s}}.
\end{align*}
By definition of $\varepsilon_V$, the first term satisfies
\[
\|\delta_{\mathrm{mix}}^V\|_{1,\mu_T^{s}}\le \varepsilon_V.
\]
For the transfer term, use the dual representation of the $L_1(\mu_T^{s})$ norm,
\[
\|\delta_T^V-\delta_{\mathrm{mix}}^V\|_{1,\mu_T^{s}}
=
\sup_{\|\varphi\|_\infty\le 1}
\left|
\mathbb E_{s\sim\mu_T^{s}}
\big[\varphi(s)(\delta_T^V(s)-\delta_{\mathrm{mix}}^V(s))\big]
\right|.
\]
Applying the weighted Bellman residual-transfer argument of Lemma~\ref{lem:bellman} with the bounded state multiplier $\varphi(s)$, and using the conditioning/Jensen step that maps the transition-level TD residual to the state residual in Eq.~\eqref{eq:def-deltaV-risk}, gives
\[
\|\delta_T^V-\delta_{\mathrm{mix}}^V\|_{1,\mu_T^{s}}
\le
\beta\,(R_1\,\Delta_m+R_2\,\Delta_w)+4\beta\,\varepsilon_H.
\]
Combining the two terms yields Eq.~\eqref{eq:target-risk-V}, which is exactly the Bellman-risk bound under the absolute-loss choice in Eq.~\eqref{eq:bellman-risk-def}.
\end{proof}

\section{Variational characterization of gate--then--weight fusion}
\label{sec:proof_variational}

\subsection{Proof of Proposition~\ref{prop:optimal-gate}}
\begin{proof}
Let $\mathcal I_{n_{\mathrm{keep}}}:=\{I\in\{0,1\}^N:\sum_{i=1}^N I_i=n_{\mathrm{keep}}\}$ and write $d_i^m:=d^m(\tau_i^S)$. The objective of Eq.~\eqref{eq:gate-subproblem} can be written as
\[
\Delta_m(I)=\max_{i:I_i=1} d_i^m.
\]
Take any feasible gate $I\in\mathcal I_{n_{\mathrm{keep}}}$. Suppose there exist indices $i,j$ such that $I_i=1$, $I_j=0$, and $d_i^m>d_j^m$. Define a new feasible gate $\widetilde I$ by swapping the two decisions: $\widetilde I_i=0$, $\widetilde I_j=1$, and $\widetilde I_\ell=I_\ell$ for $\ell\notin\{i,j\}$. Then
\begin{align*}
\Delta_m(\widetilde I)
&=
\max\!\Big(\max_{\ell:\,I_\ell=1,\,\ell\neq i} d_\ell^m,\ d_j^m\Big) \\
&\le
\max\!\Big(\max_{\ell:\,I_\ell=1,\,\ell\neq i} d_\ell^m,\ d_i^m\Big)
=
\Delta_m(I).
\end{align*}
Thus exchanging a retained fragment with larger score for a discarded fragment with smaller score cannot increase the objective. Repeating this exchange argument until no such inversion remains yields a gate that retains exactly the $n_{\mathrm{keep}}$ smallest values of $d_i^m$. Any such gate attains the minimum possible maximum mismatch, namely the $n_{\mathrm{keep}}$-th order statistic $q_{n_{\mathrm{keep}}}$. Therefore any optimizer has the threshold form
\[
I_i^{\star}=\mathbf 1\!\big(d_i^m\le q_{n_{\mathrm{keep}}}\big),
\]
up to arbitrary tie-breaking among fragments with score equal to $q_r$.
\end{proof}

\subsection{Proof of Proposition~\ref{prop:optimal-weight}}
\begin{proof}
Fix a gate $I$ and abbreviate $\mathcal G:=\mathcal G(I)$. Write the optimization problem as
\[
\min_{w_i\ge 0,\,\sum_{i\in\mathcal G}w_i=1}
\lambda_w\sum_{i\in\mathcal G} w_i c_i
+\tau_{\mathrm{ent}}\sum_{i\in\mathcal G} w_i\log\frac{w_i}{u_I(i)}.
\]
Because the KL term is strictly convex on the simplex and the linear cost is convex, the objective is strictly convex and thus admits a unique minimizer. Introduce the Lagrangian
\[
\mathcal L(w,\nu)
=
\lambda_w\sum_{i\in\mathcal G} w_i c_i
+\tau_{\mathrm{ent}}\sum_{i\in\mathcal G} w_i\log\frac{w_i}{u_I(i)}
+\nu\Big(\sum_{i\in\mathcal G} w_i-1\Big).
\]
For any interior optimum, stationarity gives
\[
\frac{\partial \mathcal L}{\partial w_i}
=
\lambda_w c_i+\tau_{\mathrm{ent}}\Big(\log\frac{w_i}{u_I(i)}+1\Big)+\nu
=0,
\qquad i\in\mathcal G.
\]
Solving for $w_i$ yields
\begin{align*}
&\log\frac{w_i}{u_I(i)}
=-\frac{\lambda_w c_i}{\tau_{\mathrm{ent}}}-1-\frac{\nu}{\tau_{\mathrm{ent}}} \\
&\quad\Longrightarrow\quad
w_i
=
\nu'\,u_I(i)\exp(-\lambda_w c_i/\tau_{\mathrm{ent}}),
\end{align*}
where $\nu':=\exp(-1-\nu/\tau_{\mathrm{ent}})$ is a normalization constant independent of $i$. Enforcing $\sum_{i\in\mathcal G} w_i=1$ gives
\[
\nu'=
\Big(\sum_{j\in\mathcal G}u_I(j)\exp(-\lambda_w c_j/\tau_{\mathrm{ent}})\Big)^{-1},
\]
which proves Eq.~\eqref{eq:gibbs-weights}. Since the objective is strictly convex, this optimizer is unique. Setting $d_i^w:=-c_i$ and $\eta_w:=\lambda_w/\tau_{\mathrm{ent}}$ yields the exponential form used in TAF.
\end{proof}

\subsection{Proof of Theorem~\ref{thm:variational-dfdt}}
\begin{proof}
Let $\lambda_m=\beta R_1$ and $\lambda_w=\beta R_2$. Theorem~\ref{thm:bellman-risk} implies that for any admissible pair $(I,w)$,
\[
\mathfrak{R}_{\mathrm{Bell}}\!\big({I,w}\big)
\le
\varepsilon_V+\lambda_m\Delta_m(I)+\lambda_w\Delta_w(I,w)+4\beta\varepsilon_H.
\]
By the samplewise majorization condition in Eq.~\eqref{eq:samplewise-majorizer}, with $c_i$ denoting the chosen population cost $c_i^{\kappa}$,
\[
\Delta_w(I,w)
\le
\sum_{i\in\mathcal G(I)} w_i c_i.
\]
For completeness, one sufficient construction is the product coupling $\Pi_w:=\sum_{i\in\mathcal G(I)}w_i\,\delta_{u_i}\otimes\mathbb P_T$, whose first marginal is $\mathbb P_S^{I,w}$ and whose second marginal is $\mathbb P_T$. Its transport cost is $\sum_i w_i\int C(u_i,v){\rm d}\mathbb P_T(v)$, so the Wasserstein optimum is no larger than this weighted samplewise cost. Substituting the majorizer gives
\[
\mathfrak{R}_{\mathrm{Bell}}\!\big({I,w}\big)
\le
\varepsilon_V+\lambda_m\Delta_m(I)
+\lambda_w\sum_{i\in\mathcal G(I)} w_i c_i
+4\beta\varepsilon_H.
\]
Therefore the Bellman risk is upper-bounded by the unregularized part of Eq.~\eqref{eq:variational-surrogate}. Adding the nonnegative entropy term gives the regularized surrogate $\mathcal J_{\tau_{\mathrm{ent}}}(I,w)$. This extra term is not required by the Bellman-risk upper bound itself; it regularizes the within-support transport minimization so that the optimizer does not collapse onto a single lowest-cost sample and instead remains close to the reference law $u_I$.

Items (1) and (2) in the theorem follow directly from Proposition~\ref{prop:optimal-gate} and Proposition~\ref{prop:optimal-weight}, respectively, because Eq.~\eqref{eq:sequential-surrogate-program} is the sequential combination of those two subproblems. Item (3), namely Eq.~\eqref{eq:variational-risk-bound}, is the preceding Bellman-risk upper bound evaluated at $(I^{\star},w^{\star})$.
\end{proof}

\subsection{Population-to-empirical transition}
\label{sec:population_empirical_transition}

\begin{proposition}[Population-to-empirical stability of the stagewise TAF optimizer]
\label{prop:population-empirical-transition}
Let $\hat I$ and $\hat w$ denote the empirical structural gate and empirical within-support weights obtained by replacing $d_i^m$ and $c_i$ in Eqs.~\eqref{eq:gate-subproblem} and \eqref{eq:weight-subproblem} by $\hat d_i^m$ and $\hat c_i$, respectively. If these empirical scores satisfy
\[
\sup_i |\hat d_i^m-d_i^m|\le \epsilon_m,
\qquad
\sup_i |\hat c_i-c_i|\le \epsilon_w,
\]
then the resulting empirical gate and weights satisfy the stagewise population suboptimality bounds
\begin{equation}
\label{eq:empirical-gate-bound}
\Delta_m(\hat I)
\le
\min_{I:\,\sum_i I_i=n_{\mathrm{keep}}}\Delta_m(I)+2\epsilon_m,
\end{equation}
and
\begin{equation}
\label{eq:empirical-weight-bound}
\begin{aligned}
&\lambda_w\!\sum_{i\in\mathcal G(\hat I)} \hat w_i c_i
+\tau_{\mathrm{ent}}\,\mathrm{KL}(\hat w\|u_{\hat I})
\\
&\le
\inf_{w}
\Big\{
\lambda_w\!\sum_{i\in\mathcal G(\hat I)} w_i c_i
+\tau_{\mathrm{ent}}\,\mathrm{KL}(w\|u_{\hat I})
\Big\}
+2\lambda_w\epsilon_w.
\end{aligned}
\end{equation}
\end{proposition}
\begin{proof}
Let $\widehat\Delta_m(I):=\max_{i:I_i=1}\hat d_i^m$. The uniform error assumption $\sup_i|\hat d_i^m-d_i^m|\le \epsilon_m$ implies
\[
|\widehat\Delta_m(I)-\Delta_m(I)|\le \epsilon_m
\qquad\text{for every feasible }I.
\]
Let $I^{\star}$ minimize $\Delta_m(I)$ and let $\hat I$ minimize $\widehat\Delta_m(I)$ over the budget-constrained gate set. Then
\begin{align*}
\Delta_m(\hat I)
&\le \widehat\Delta_m(\hat I)+\epsilon_m
\le \widehat\Delta_m(I^{\star})+\epsilon_m \\
&\le \Delta_m(I^{\star})+2\epsilon_m
=
\min_{I:\,\sum_i I_i=n_{\mathrm{keep}}}\Delta_m(I)+2\epsilon_m,
\end{align*}
which proves Eq.~\eqref{eq:empirical-gate-bound}.

For the weight stage, fix the empirical gate $\hat I$ and define the population and empirical objectives on $\Delta(\mathcal G(\hat I))$ by
\[
F(w):=\lambda_w\sum_{i\in\mathcal G(\hat I)} w_i c_i+\tau_{\mathrm{ent}}\,\mathrm{KL}(w\|u_{\hat I}),
\]
\[
\widehat F(w):=\lambda_w\sum_{i\in\mathcal G(\hat I)} w_i \hat c_i+\tau_{\mathrm{ent}}\,\mathrm{KL}(w\|u_{\hat I}).
\]
Since $\sum_i w_i=1$ and $\sup_i|\hat c_i-c_i|\le \epsilon_w$,
\[
|\widehat F(w)-F(w)|
\le
\lambda_w\sum_i w_i\,|\hat c_i-c_i|
\le
\lambda_w\epsilon_w,
\ \  \forall w\in\Delta(\mathcal G(\hat I)).
\]
Let $w_{\hat I}^{\star}$ minimize $F$ on $\Delta(\mathcal G(\hat I))$, and let $\hat w$ minimize $\widehat F$ on the same simplex. Then
\begin{align*}
F(\hat w)
&\le \widehat F(\hat w)+\lambda_w\epsilon_w
\le \widehat F(w_{\hat I}^{\star})+\lambda_w\epsilon_w \\
&\le F(w_{\hat I}^{\star})+2\lambda_w\epsilon_w
=
\inf_{w\in\Delta(\mathcal G(\hat I))}F(w)+2\lambda_w\epsilon_w,
\end{align*}
which is exactly Eq.~\eqref{eq:empirical-weight-bound}.
\end{proof}

\section{Algorithm details of TAF-DT}\label{sec:alg-details}
\begin{algorithm*}[ht]
\caption{TAF-DT Training}
\label{alg:train}
\begin{algorithmic}[1]
\REQUIRE Source dataset $\mathcal D_{\rm src}$, target dataset $\mathcal D_{\rm tar}$, keep ratio $\xi$, context length $K$, batch size $N$, target-update rate $\eta_{\rm exp}$
\STATE Initialize DT policy $\pi_\theta$, twin critics $Q_{\phi_1},Q_{\phi_2}$, target critics $Q_{\phi'_1},Q_{\phi'_2}$, auxiliary token critics $V_\varphi,Q_\psi$, and command network $C_\omega$
\STATE Pre-compute MMD discrepancies $\{d_i^m\}$ by Eq.~\eqref{eq:mmd_cost} for source fragments and admit the top-$\xi\%$ fragments to obtain the gate $I_i=\mathbf 1(d_i^m\le q_\xi)$
\STATE Pre-compute reference OT-row costs $\{\hat c_i\}$ on the retained support, calibrate them to $\{\widetilde c_i\}$ via Eq.~\eqref{eq:norm-ot-dev}, set $d_i^w=-\widetilde c_i$ by Eq.~\eqref{eq:ot-dist}, and form source weights $w_i\propto u_I(i)\exp(\eta_w d_i^w)$
\STATE Construct the fused empirical sampler corresponding to $\mathbb P_{\mathrm{mix}}^{I_m,w}=(1-\beta)\mathbb P_T+\beta\mathbb P_S^{I_m,w}$
\STATE Train the auxiliary critics $V_\varphi,Q_\psi$ on $\mathbb P_{\mathrm{mix}}^{I_m,w}$ using Eqs.~\eqref{V_loss} and \eqref{Q_loss}
\STATE Relabel the fused trajectories with advantage tokens $A_t=Q_\psi(s_t,a_t)-V_\varphi(s_t)$ and train the command network $C_\omega$
\FOR{each gradient step}
    \STATE Sample an advantage-conditioned mini-batch $x_{t-K+1:t}\sim \mathbb P_{\mathrm{mix}}^{I_m,w}$
    \STATE Update the twin critics $Q_{\phi_1},Q_{\phi_2}$ by minimizing Eq.~\eqref{eq:q_loss}
    \STATE Update the target critics via $\phi'_j\leftarrow\eta_{\rm exp}\phi_j+(1-\eta_{\rm exp})\phi'_j$ for $j\in\{1,2\}$
    \STATE Update the DT policy $\pi_\theta$ by minimizing Eq.~\eqref{eq:pi_loss}
\ENDFOR
\end{algorithmic}
\end{algorithm*}

\begin{algorithm*}[ht]
\caption{TAF-DT Inference}
\label{alg:infer}
\begin{algorithmic}[1]
\REQUIRE Trained DT policy $\pi_{\theta}$, trained command network $C_\omega$, sequence length $K$, (optional) normalization stats $(\mu_A,\sigma_A)$ from training, environment $\mathcal{M}_T$
\STATE \textcolor{blue}{// No critics or OT/MMD are needed at test time. We only use $C_\omega$ to produce command tokens and $\pi_\theta$ to act.}
\STATE Initialise circular buffers for the last $K$ tokens:
\[
\mathsf{S}\leftarrow[\ ],\quad \mathsf{A}\leftarrow[\ ],\quad \mathsf{C}\leftarrow[\ ],\quad \mathsf{T}\leftarrow[\ ],\quad \mathsf{M}\leftarrow[\ ]
\]
\STATE Reset environment; receive initial state $s_1$ and set $t\leftarrow 1$
\WHILE{episode not terminal}
    \STATE \textcolor{blue}{// Compute command token from the current state}
    \STATE $c_t^{\rm raw}\leftarrow C_\omega(s_t)$
    \STATE \textbf{if} training used standardized advantages (cf.\ Eq.\,(\ref{eq:adv-std})) \textbf{then} $c_t\leftarrow c_t^{\rm raw}$ \textbf{else} $c_t\leftarrow \dfrac{c_t^{\rm raw}-\mu_A}{\sigma_A+\varepsilon}$ \textbf{end if}
    \STATE \textcolor{blue}{// Update rolling context (pad left with zeros and mask invalid tokens)}
    \STATE Append $s_t$ to $\mathsf{S}$, $c_t$ to $\mathsf{C}$, $t$ to $\mathsf{T}$, and $1$ to $\mathsf{M}$; keep only the last $K$ entries of each
    \STATE Let $\mathsf{S}_{t-K+1:t}$, $\mathsf{C}_{t-K+1:t}$, $\mathsf{T}_{t-K+1:t}$, and $\mathsf{M}_{t-K+1:t}$ be the length-$K$ sequences after left-padding with zeros
    \STATE Define $\mathsf{A}_{t-K+1:t-1}$ as the last $K-1$ actions (left-padded with zeros); if $t=1$ then $\mathsf{A}_{t-K+1:t-1}$ is all zeros and the first mask entries in $\mathsf{M}$ are $0$
    \STATE \textcolor{blue}{// Policy inference with command-conditioned tokens}
    \STATE $a_t \leftarrow \pi_{\theta}\!\big(\mathsf{S}_{t-K+1:t}, \mathsf{A}_{t-K+1:t-1}, \mathsf{C}_{t-K+1:t}, \mathsf{T}_{t-K+1:t}, \mathsf{M}_{t-K+1:t}\big)$
    \STATE Execute $a_t$ in $\mathcal{M}_T$; observe $(r_t, s_{t+1})$
    \STATE Append $a_t$ to $\mathsf{A}$ (keep the last $K-1$ entries); set $t\leftarrow t+1$ and $s_t\leftarrow s_{t+1}$
\ENDWHILE
\STATE \textbf{return} trajectory $\tau=\{(s_t,a_t,r_t)\}_{t=1}^{T}$
\end{algorithmic}
\end{algorithm*}

\textbf{Computing method of normalized scores.}
Because raw returns are not directly comparable across environments, we follow D4RL \citep{fu2020d4rl} and report the Normalized Score (NS):
\begin{equation}\label{eq:normalized-score}
\mathrm{NS} \;=\; \frac{\hat J - \hat J_{\text{rand}}}{\hat J_{\text{exp}} - \hat J_{\text{rand}}} \times 100,
\end{equation}
where \(\hat J\) is the empirical return of the learned policy, \(\hat J_{\text{exp}}\) is the expert policy’s empirical return, and \(\hat J_{\text{rand}}\) is the empirical return of a random policy. By construction, \(\mathrm{NS}=100\) corresponds to expert-level performance and \(\mathrm{NS}=0\) corresponds to random performance. See Appendix~C.1 of~\cite{lyu2025cross} for dataset details about $\hat J_{\text{rand}}$ and $\hat J_{\text{exp}}$.

\textbf{Normalisation of OT-based costs.}
To make the empirical OT costs $\hat c_i$ numerically stable across tasks and batches, we apply a min--max cost calibration on the retained support:
\begin{equation}\label{eq:norm-ot-dev}
\widetilde c_i
=
\frac{\hat c_i-\min_{j\in\mathcal G(I)}\hat c_j}{\max_{j\in\mathcal G(I)}\hat c_j-\min_{j\in\mathcal G(I)}\hat c_j+\varepsilon}.
\end{equation}
This mapping gives $\widetilde c_i\in[0,1]$ up to the small stabilizer $\varepsilon$ and preserves the ordering induced by $\hat c_i$. We then set the implementation feasibility score $d_i^w=-\widetilde c_i\in[-1,0]$, so the corresponding unnormalized implementation weights
\begin{equation}\label{eq:exp-weight}
\tilde w_i\ :=\ \exp(\eta_w\,d_i^{w})\ \in\ \big[e^{-\eta_w},\,1\big]
\end{equation}
are bounded, preventing gradient explosion while still down-weighting OT-distant source fragments. After multiplying by the gate $I_i$, these factors are renormalized on the retained support to realize the source distribution $\mathbb P_S^{I_m,w}$. Thus, $\tilde w_i$ is the practical counterpart of the Gibbs factor in Eq.~\eqref{eq:ot-weight}, and Eq.~\eqref{eq:norm-ot-dev} makes cost-based weighting scale-free across domains and robust to outliers in $\hat c$.

\textbf{Command network $C_\omega$ trained via expectile regression.}
The command network $C_\omega$ outputs an \emph{advantage-consistent command token} that replaces RTG during inference. We first form per-token advantages from the auxiliary value estimators,
\begin{equation}\label{eq:adv}
A_i\ :=\ Q_\psi(s_i,a_i)\ -\ V_\varphi(s_i),
\end{equation}
and optionally standardize them to improve numerical stability:
\begin{equation}\label{eq:adv-std}
\tilde A_i\ :=\ \frac{A_i-\mu_A}{\sigma_A+\varepsilon},
\end{equation}
where $(\mu_A,\sigma_A)$ are the mean and standard deviation of $\{A_i\}$. Define the empirical sample weight
\[
\bar\omega_i:=
\begin{cases}
1, & u_i\in\mathcal D_{\rm tar},\\
I_i\,\tilde w_i, & u_i\in\mathcal D_{\rm src},
\end{cases}
\]
with $I_i$ the MMD gate and $\tilde w_i$ from Eq.~\eqref{eq:exp-weight}. We then train $C_\omega$ to predict a high-expectile summary of the state-conditional advantage distribution using the asymmetric least-squares loss
\begin{equation}\label{eq:expectile-loss}
\mathcal L_C(\omega)
=
\mathbb E_{\mathcal D_{\rm tar}\cup\mathcal D_{\rm src}}
\Big[
\bar\omega_i\,\rho_\zeta\!\big(\tilde A_i-C_\omega(s_i)\big)
\Big],
\end{equation}
where $\zeta\in(0.5,1)$ (e.g., $0.7\!\sim\!0.9$) controls the degree of optimism. The weighting $\bar\omega_i$ ensures that target tokens are always used, while source tokens contribute only when they survive the MMD gate and receive a sufficiently large OT-based feasibility weight.

\textbf{Usage of $C_\omega$ at inference.}
At test time, we compute a scalar command $c_t = C_\omega(s_t)$ from the current state and feed it as the conditioning token to the DT in place of RTG. Because $c_t$ is derived from standardized advantages rather than raw returns, it offers a return-distribution- and horizon-robust guidance signal that is shared across domains. This stabilizes token-level conditioning under dynamics shifts and reduces stitching artefacts caused by dynamics-induced return-distribution and effective-horizon changes.

\begin{table}[htb]
    \centering
    \caption{Default hyperparameter setup for TAF-DT. For exact hyperparameter setups for each experiment, please refer to the code base.}
    \label{tab:hyperparameters}
    \begin{tabular}{ll}
        \toprule
        \textbf{Hyperparameter} & \textbf{Value} \\
        \midrule
        Number of layers & $4$ \\
        Number of attention heads & $4$ \\
        Embedding dimension & $256$ \\
        Context length $K$ & $\{5, 10, 20\}$ \\
        Dropout & $0.1$ \\
        Learning rate & $3 \times 10^{-4}$ \\
        Optimizer & Adam \\
        Discount factor & $0.99$ \\
        Nonlinearity & ReLU \\
        Target update rate & $5 \times 10^{-3}$ \\
        \midrule
        Pretrained Q network hidden size & $(256, 256, 256)$ \\
        Pretrained V network hidden size & $(256, 256, 256)$ \\
        Command network hidden size & $(256, 256, 256)$ \\
        Number of sampled latent variables $M$ & $10$ \\
        Standard deviation of Gaussian distribution & $\sqrt{0.1}$ \\
        OT Cost function & cosine \\
        MMD fragment keep ratio $\xi_{\rm MMD}\%$ & $50\%$ \\
        Policy regularization coefficient $\eta_{\rm reg}$  & $\{0.3, 0.4, 0.5, 0.6\}$ \\
        Source domain Batch size & $64$ \\
        Target domain Batch size & $128$ \\
        \bottomrule
    \end{tabular}
\end{table}

\section{Environment Details and Setup}
\label{app:env-setup}

\subsection{Environment Specifications}
\label{app:env-specs}

Following recent work in offline reinforcement learning \citep{fu2020d4rl} and cross-domain policy adaptation \citep{lyu2025cross}, we evaluate TAF-DT on four continuous control tasks from the MuJoCo simulator \citep{todorov2012mujoco}: \texttt{HalfCheetah-v2}, \texttt{Hopper-v2}, \texttt{Walker2d-v2}, and \texttt{Ant-v2}. All environments use the MuJoCo v2 physics engine via OpenAI Gym \citep{brockman2016openai}.

We use the standard D4RL configuration with a maximum episode length of 1000 timesteps for all tasks. Episodes terminate early only if the agent enters a failure state (e.g., the Hopper falls over). We adopt the original D4RL reward specifications without modification, including implicit survival bonuses and center-of-mass velocity components where applicable.

\textbf{Transformer context length.} Due to varying trajectory structure across tasks, we use set task-specific context windows. For morphology and kinematic shifts, we used $K=20$ for the HalfCheetah task and $K=5$ for the Hopper, Walker2D and Ant tasks. For gravity shifts, we used $K=10$ for the HalfCheetah task, $K=20$ for the Hopper and Walker2D task, and $K=5$ for the Ant task. These results help balance model capacity and empirical performance in dynamic shifts.

\subsection{Evaluation protocol} 

\textbf{Normalized score.} For each trained policy, we report the normalized score (NS) averaged over 10 evaluation episodes using deterministic actions (mean of the policy distribution). The normalized score is computed as:
\begin{equation}
\mathrm{NS} = \frac{\hat{J} - \hat{J}_{\text{rand}}}{\hat{J}_{\text{exp}} - \hat{J}_{\text{rand}}} \times 100,
\end{equation}
where $\hat{J}$ is the empirical return of the learned policy, $\hat{J}_{\text{exp}}$ is the expert policy's return in the \emph{target domain}, and $\hat{J}_{\text{rand}}$ is the random policy's return in the target domain. Reference values for $\hat{J}_{\text{exp}}$ and $\hat{J}_{\text{rand}}$ under each domain shift are provided in Appendix~C.1 of~\cite{lyu2025cross}.

\begin{table*}[ht]
\caption{\textbf{Performance comparison of cross-domain offline RL algorithms given gravity shifts.} The meanings of each abbreviation are the same as those listed in Table~\ref{tab:morph-shift}. We report \emph{normalized} target-domain performance (\emph{mean $\pm$ std.}) over \textbf{five} seeds; We \textbf{bold} and highlight the best cell.}
\label{tab:gravity-shift}
\centering
\small

\begin{tabular}{
  @{}ll
  | >{$}c<{$}                             
  | >{$}c<{$} >{$}c<{$} >{$}c<{$}         
  | >{$}c<{$} >{$}c<{$} >{$}c<{$}         
  | >{$}c<{$}                             
  @{}}
\toprule
Source & Target & \text{IQL} & \text{DARA} & \text{IGDF} & \text{OTDF} & \text{DT} & \text{QT} & \text{DADT} & \text{TAF-DT} \\
\midrule
half-m   & medium         & 39.6 & \bestplain{41.2} & 36.6 & 40.7 & 28.4 & 40.2 & 36.6 & 7.3\!\pm\!4.3 \\
half-m   & medium-expert  & 39.6 & 40.7 & 38.7 & 28.6 & 45.1 & \bestplain{62.1} & 34.7 & 7.8\!\pm\!2.4 \\
half-m   & expert         & 42.4 & 39.8 & 39.6 & 36.1 & 41.8 & \bestplain{49.1} & 45.7 & 13.8\!\pm\!11.7 \\
half-m-r & medium         & 20.1 & 17.6 & 14.4 & 21.5 & 18.3 & \bestplain{51.6} & 25.3 & 5.9\!\pm\!2.5 \\
half-m-r & medium-expert  & 17.2 & 20.2 & 10.0 & 14.7 & 17.2 & 2.1  & \bestplain{27.1} & 5.7\!\pm\!2.4 \\
half-m-r & expert         & 20.7 & 22.4 & 15.3 & 11.4 & 7.8  & 2.5  & \bestplain{23.6} & 17.9\!\pm\!10.0 \\
half-m-e & medium         & 38.6 & 37.8 & 37.7 & 39.5 & 35.1 & \bestplain{69.3} & 44.0 & 5.6\!\pm\!2.6 \\
half-m-e & medium-expert  & 39.6 & 39.4 & 40.7 & 32.4 & 38.2 & \bestplain{67.0} & 32.0 & 6.0\!\pm\!2.9 \\
half-m-e & expert         & 43.4 & 45.3 & 41.1 & 26.5 & 40.7 & \bestplain{68.5} & 37.8 & 21.9\!\pm\!8.3 \\
\midrule
hopp-m   & medium         & 11.2 & 17.3 & 15.3 & 32.4 & 19.7 & 16.1 & 12.8 & \bestcell{82.4}{7.5} \\
hopp-m   & medium-expert  & 14.7 & 15.4 & 15.1 & 24.2 & 11.6 & 12.8 & 11.6 & \bestcell{56.7}{23.3} \\
hopp-m   & expert         & 12.5 & 19.3 & 14.8 & \bestplain{33.7} & 11.0 & 12.3 & 12.7 & 22.7\!\pm\!11.2 \\
hopp-m-r & medium         & 13.9 & 10.7 & 15.3 & 31.1 & 14.2 & 19.9 & 22.6 & \bestcell{58.8}{27.5} \\
hopp-m-r & medium-expert  & 13.3 & 12.5 & 15.4 & 24.2 & 13.7 & 22.3 & 16.6 & \bestcell{66.4}{17.7} \\
hopp-m-r & expert         & 11.0 & 14.3 & 16.1 & 31.0 & 19.6 & 18.7 & 21.5 & \bestcell{42.4}{16.6} \\
hopp-m-e & medium         & 19.1 & 18.5 & 22.3 & 26.4 & 13.0 & 14.3 & 11.6 & \bestcell{63.6}{18.0} \\
hopp-m-e & medium-expert  & 16.8 & 16.0 & 16.6 & 28.3 & 13.6 & 14.4 & 11.7 & \bestcell{39.2}{27.8} \\
hopp-m-e & expert         & 20.9 & 23.9 & 26.0 & 44.9 & 13.1 & 14.0 & 13.2 & \bestcell{75.4}{19.0} \\
\midrule
walk-m   & medium         & 28.1 & 28.4 & 22.1 & 36.6 & 36.2 & 29.5 & 37.4 & \bestcell{43.1}{7.2} \\
walk-m   & medium-expert  & 35.7 & 30.7 & 35.4 & 44.8 & 38.2 & \bestplain{45.2} & 29.1 & 21.5\!\pm\!4.5 \\
walk-m   & expert         & 37.3 & 36.0 & 36.2 & 44.0 & 46.4 & 44.0 & \bestplain{54.0} & 22.6\!\pm\!5.7 \\
walk-m-r & medium         & 14.6 & 14.1 & 11.6 & 32.7 & 28.6 & 18.9 & 24.8 & \bestcell{44.1}{2.9} \\
walk-m-r & medium-expert  & 15.3 & 15.9 & 13.9 & \bestplain{31.6} & 26.9 & 20.0 & 29.8 & 22.7\!\pm\!7.0 \\
walk-m-r & expert         & 15.8 & 15.7 & 15.2 & \bestplain{31.3} & 28.0 & 28.6 & 20.1 & 26.7\!\pm\!11.8 \\
walk-m-e & medium         & 39.9 & 41.6 & 33.8 & 30.2 & 42.5 & \bestplain{56.7} & 45.5 & 41.4\!\pm\!3.2 \\
walk-m-e & medium-expert  & 49.1 & 45.8 & 44.7 & 53.3 & 39.4 & \bestplain{55.8} & 30.6 & 23.6\!\pm\!5.1 \\
walk-m-e & expert         & 40.4 & 56.4 & 45.3 & \bestplain{61.1} & 39.6 & 47.4 & 34.5 & 23.6\!\pm\!8.9 \\
\midrule
ant-m    & medium         & 10.2 & 9.4  & 11.3 & 45.1 & 22.0 & 15.3 & 12.4 & \bestcell{61.0}{8.7} \\
ant-m    & medium-expert  & 9.4  & 10.0 & 9.4  & 33.9 & 17.7 & 14.1 & 14.0 & \bestcell{52.8}{15.7} \\
ant-m    & expert         & 10.2 & 9.8  & 9.7  & 33.2 & 18.9 & 15.7 & 13.7 & \bestcell{58.3}{5.8} \\
ant-m-r  & medium         & 18.9 & 21.7 & 19.6 & 29.6 & 18.8 & 13.9 & 21.4 & \bestcell{66.9}{8.5} \\
ant-m-r  & medium-expert  & 19.1 & 18.3 & 20.3 & 25.4 & 13.9 & 13.6 & 18.5 & \bestcell{44.9}{5.5} \\
ant-m-r  & expert         & 18.5 & 20.0 & 18.8 & 24.5 & 14.6 & 10.6 & 17.7 & \bestcell{38.8}{11.1} \\
ant-m-e  & medium         & 9.8  & 8.1  & 8.9  & 18.6 & 11.3 & 11.6 & 20.6 & \bestcell{68.9}{6.4} \\
ant-m-e  & medium-expert  & 9.0  & 6.4  & 7.2  & 34.0 & 18.0 & 12.2 & 15.2 & \bestcell{45.7}{18.0} \\
ant-m-e  & expert         & 9.1  & 10.4 & 9.2  & 23.2 & 11.6 & 10.0 & 15.3 & \bestcell{41.2}{11.9} \\
\midrule
\multicolumn{2}{@{}c|}{Total Score} &
  825.0 & 851.0 & 803.6 & 1160.7 & 874.7 & 1020.3 & 895.7 & \bestplain{1347.3} \\
\bottomrule
\end{tabular}
\end{table*}

\textbf{Random seeds.} 
For main experiments (Tables \ref{tab:morph-shift},  \ref{tab:kinematic-shift}, \ref{tab:gravity-shift}), we evaluate policies using the final checkpoint from each of 5 random seeds (1, 2, 3, 4, 5) and report the mean ± standard deviation across seeds. Ablation studies use different evaluation protocols as specified in individual table captions.

\textbf{Computational resources.} 
Each experiment is trained on a single NVIDIA RTX~4090 GPU (24\,GB). The total training time per task is typically 6--7 hours for 100{,}000 gradient steps (see Appendix Section~\ref{sec:training-time} for detailed training time analyses).

\textbf{Datasets.} 
Our source-domain datasets are drawn from the D4RL benchmark~\citep{fu2020d4rl}. Target-domain datasets are constructed under three types of dynamics shifts: \textbf{gravity shifts}, \textbf{kinematic shifts}, and \textbf{morphology shifts}. The exact environment modifications (XML parameter changes, joint range restrictions, body geometry alterations) follow Appendix~C.2--C.4 of~\cite{lyu2025cross}. We refer readers to their work for detailed XML code and implementation specifics.

\begin{table*}[ht]
\centering
\caption{Ablation study on the expectile coefficient for TAF-DT. Results show the mean and standard deviation over the last 3 checkpoints from a single seed.}
\label{tab:expectile_ablation}
\begin{tabular}{lll|cccc}
\toprule
Source  & Target & Shift & $\zeta=0.5$ & $\zeta=0.55$ & $\zeta=0.7$ & $\zeta=0.98$ \\
\midrule
half-m-e & medium & kinematic & $41.5\pm0.2$ & $41.5\pm0.3$ & $\textbf{41.7}\pm0.4$ & $41.5\pm0.3$ \\
ant-m & medium & gravity & $\textbf{61.8}\pm4.3$ & $58\pm2.9$ & $60.5\pm16.3$ & $58.5\pm3.8$ \\
walk-m & medium-expert & morph &  $31.1\pm16.1$ & $32.6\pm3.4$ & $39.7\pm6.6$ & $\textbf{41.1}\pm5.3$ \\
hopp-m-e & medium & kinematic & $61.9\pm3.0$ & $61.4\pm2.7$ & $61.0\pm3.3$ & $\textbf{62.2}\pm2.7$ \\
\bottomrule
\end{tabular}
\end{table*}

\begin{table*}[htb]
\centering
\caption{Ablation study on the 1-dim $r$ channel of OT computation. For each seed, we compute the mean over the last 3 checkpoints; reported values are mean $\pm$ std across the 3 per-seed means.}
\label{tab:ot_r_ablation}
\begin{tabular}{lll|cc}
\toprule
Source   & Target        & Shift     & OT w the $r$ channel & OT w/o the $r$ channel  \\
\midrule
ant-m    & medium        & kinematic & $53.1\pm6.4$	& $\textbf{56.9}\pm3.1$\\
hopp-m    & medium        & kinematic & $59.3\pm11.4$ & $\textbf{63.7}\pm3.6$\\
hopp-m    & medium          & morph & $\textbf{57.8}\pm11.9$ & $53.0\pm13.4$ \\
half-m-r & medium-expert & morph & $43.8\pm0.8$ & $\textbf{44.6}\pm0.4$ \\
walk-m-r & medium & gravity & $40.7\pm4.1$ & $\textbf{44.3}\pm0.8$ \\
\bottomrule
\end{tabular}
\end{table*}

\textbf{State normalization.}
For all TAF-DT experiments, we apply per-dimension normalization to state observations:
\begin{equation}
\hat{s}_i = \frac{s_i - \mu_i}{\sigma_i + \epsilon},
\end{equation}
where $\mu_i$ and $\sigma_i$ are the mean and standard deviation of dimension $i$ 
computed over all states in the  target training dataset, 
and $\epsilon = 10^{-8}$ prevents division by zero.
This normalization is applied identically during training and evaluation.

We normalize states but \emph{not} actions or rewards. 
The normalization statistics are computed once before training and remain fixed 
throughout all experiments for a given source-target pair.

\textbf{Two-stage gate--then--weight preprocessing.}
We process source data in two sequential stages to balance data utilization and cross-domain compatibility: an MMD fragment gate fixes the retained support, and an OT-induced soft weighting rule determines how strongly each retained transition contributes.

\textbf{Cost computations before the main training loop.}
Before training begins, we pre-compute the following quantities:
\begin{itemize}
    \item Maximum Mean Discrepancy (MMD) between each source state fragment
    \[
    (s_t^S,s_{t+1}^S,\ldots,s_{t+K}^S)
    \]
    and each target state fragment
    \[
    (s_t^T,s_{t+1}^T,\ldots,s_{t+K}^T),
    \]
    using the same latent state encoder $z=f_\phi(s)$ as in Eq.~\eqref{eq:mmd_cost}. This keeps the appendix consistent with the state-structure gate used in the main formulation.
    \item Reference optimal-transport (OT) row costs $\hat c_i$ between retained source transitions $(s_{t}^{S}, a_{t}^{S}, r_{t}^{S}, s_{t+1}^{S})$ and target transitions $(s_{t}^{T}, a_{t}^{T}, r_{t}^{T}, s_{t+1}^{T})$, followed by calibrated Gibbs weights $w_i\propto\exp(-\eta_w\widetilde c_i)$.
\end{itemize}

Thus, the two preprocessing scores use different granularities for different purposes: MMD uses latent state fragments to measure state-structure compatibility, whereas OT uses transition features to measure local action--transition feasibility.

\textbf{Stage 1 (MMD filtering for fragment selection).}
We first rank source fragments by their MMD distance and retain only the $\xi_{\rm MMD}\%=50\%$ with the lowest MMD, yielding a pre-filtered source dataset $\mathcal{D}_{\text{src}}^{\text{MMD}}$.

\textbf{Stage 2 (OT soft weighting for action feasibility).}
On $\mathcal{D}_{\text{src}}^{\text{MMD}}$, we compute the reference OT-row cost $\hat c_i$, calibrate it to $\widetilde c_i$, and assign each retained transition the Gibbs feasibility weight $w_i\propto\exp(-\eta_w\widetilde c_i)$. During training, source transitions are sampled from this weighted retained source law, or equivalently used with the corresponding source-side loss weights. We use 64 weighted source transitions and 128 target transitions per mini-batch, giving a final batch size of 192. Thus OT is used as a soft feasibility weighting rule rather than a second hard top-$\xi_{\rm OT}$ filtering stage.

\textbf{Hyperparameter overview.} 
Table~\ref{tab:hyperparameters} summarises the hyperparameters of a compact Transformer backbone for TAF-DT (multi-head attention with moderate depth, width, and context length), trained with standard optimisation and stabilisation choices (Adam, dropout, ReLU, soft target updates, and a fixed discount factor). Method-specific settings for TAF-DT include pretrained critics, a command network, an MMD fragment gate, and an OT-based soft weighting module with cosine ground cost. We adopt asymmetric batch sizes to emphasise target-domain learning while still leveraging feasibility-weighted source fragments. The defaults were chosen from small grids over context length and mixing strength and were found to be robust across seeds and tasks.

\section{Wider main experimental results under gravity shifts}\label{sec:wider-results}

We further report comprehensive results for gravity shifts in Table~\ref{tab:gravity-shift}. \textsc{TAF-DT} attains the best mean performance on \textbf{19} out of \textbf{36} tasks and achieves the highest total normalized score of \textbf{1347.3}, exceeding \textsc{IQL} by \textbf{63.3\%} (1347.3 vs.\ 825.0), the second-best approach \textsc{OTDF} by \textbf{16.1\%} (1347.3 vs.\ 1160.7), and the strong sequence baseline \textsc{QT} by \textbf{32.0\%} (1347.3 vs.\ 1020.3). The combination of total score and wins count again indicates broad robustness rather than domination by only a few tasks. Breaking down by environment family, \textsc{TAF-DT} dominates \emph{hopper} (wins \textbf{8} out of \textbf{9}) and \emph{ant} (wins \textbf{9} out of \textbf{9}), remains competitive on \emph{walker2d} (wins \textbf{2} out of \textbf{9}), while \emph{halfcheetah} is largely led by \textsc{QT}. Notably, \textsc{TAF-DT} delivers large margins in challenging settings such as \texttt{hopp-grav-me2e} ( \(\textbf{75.4}\pm 19.0\) ) and \texttt{ant-grav-m2m} ( \(\textbf{61.0}\pm 8.7\) ), where we define $\text{grav}\!=\!\text{gravity}$, $\text{kinematic}\!=\!\text{kin}$, and $\text{morph}\!=\!\text{morphology}$. Overall, these results corroborate \textsc{TAF-DT}'s aggregate strength under gravity shifts while making clear that some local rows remain favorable to other methods.

\section{Ablation experiments}\label{sec:ablation}

\subsection{Ablation setup}
For the ablation studies, we start from the default TAF-DT configuration and vary one component or hyperparameter at a time, keeping all other settings identical to the main experiments. To control computational cost, we evaluate on a small set of representative source--target pairs that cover all three shift types (kinematic, gravity, morphology); the exact tasks are listed in the corresponding tables.

Concretely, we (i) compare full TAF-DT (MMD+OT) against MMD-only and OT-only variants and sweep the MMD keep-ratio $\xi$ used by the state-structure gate; (ii) replace advantage-conditioned tokens with RTG tokens; (iii) vary the Q-regularization coefficient $\alpha$; (iv) sweep the expectile coefficient $\zeta$ used in the command network; and (v) modify the OT design by including or excluding the 1-dim reward channel $r$ and, in a preliminary study, by changing the OT cost function (cosine, Euclidean, squared Euclidean).

The fragment-filtering, advantage-token, $Q$-regularization, and token-stitching analyses are reported in the main paper (Section~\ref{sec:main_ablation}). We keep the remaining hyperparameter ablations below for completeness.

\subsection{Ablation on Expectile Coefficient}
We study the expectile coefficient $\zeta$ used in the command network’s expectile regression (Table~\ref{tab:expectile_ablation}). Across most tasks, TAF-DT is fairly insensitive to this choice. For example, on \texttt{half-m-e}$\to$\texttt{medium} (kinematic) all settings yield almost identical performance ($41.5$–$41.7$), and on \texttt{hopp-m-e}$\to$\texttt{medium} (kinematic) scores remain in a narrow band around $61$–$62$, with $\zeta=0.98$ only slightly ahead. For \texttt{ant-m}$\to$\texttt{medium} (gravity), all four values stay within a few points of the best setting $61.8\pm4.3$, showing no sharply tuned optimum.

The main sensitivity appears on \texttt{walk-m}$\to$\texttt{medium-expert} (morphology), where more optimistic expectiles substantially help: the mean return increases from $31.1\pm16.1$ at $\zeta=0.5$ to $41.1\pm5.3$ at $\zeta=0.98$. Overall, these trends indicate that TAF-DT remains robust over a broad range of $\zeta$ while moderately benefiting from higher, more optimistic values on harder morphology shifts. We therefore adopt $\zeta=0.98$ as the default, as it is best or competitive across all tested tasks while encouraging sufficiently optimistic value-based conditioning.

\subsection{Ablation of the 1-dim reward channel in the OT computation}
We ablate the effect of including the 1-dim reward channel $r$ in the OT cost (Table~\ref{tab:ot_r_ablation}). Across the five tested configurations, \emph{excluding} $r$ improves performance on four tasks and remains competitive on the fifth. For example, on \texttt{ant-m}$\to$\texttt{medium} and \texttt{hopp-m}$\to$\texttt{medium} (both kinematic), dropping $r$ yields higher returns with notably lower variance ($56.9\pm3.1$ vs.\ $53.1\pm6.4$ and $63.7\pm3.6$ vs.\ $59.3\pm11.4$), and similar gains appear on the morphology and gravity cases \texttt{half-m-r}$\to$\texttt{medium-expert} and \texttt{walk-m-r}$\to$\texttt{medium}. Only \texttt{hopp-m}$\to$\texttt{medium} (morphology) sees a modest benefit from including $r$, and both variants exhibit high variance there.

These results suggest that the structural information in $(s,a,s')$ is usually sufficient for effective OT-based alignment, while adding reward can inject extra noise, especially under kinematic shifts where action feasibility is the main concern. In our main experiments, we retain $(s,a,s',r)$ for consistency with the original design, and leave more nuanced ways of exploiting reward information (e.g., task-dependent weighting or learned costs) to future work.

\begin{table*}[ht]
\centering
\caption{Ablation study on cost functions used in OT. Results show the mean and standard deviation over the last 3 checkpoints from a single seed.}
\label{tab:ot_dist}
\begin{tabular}{lll|ccc}
\toprule
Source   & Target        & Shift   & Cosine & Euclidean & Squared Euclidean \\
\midrule
ant-m & medium & gravity & $\textbf{60.8}\pm13.6$ & $59.6\pm9.1$ & $\textbf{60.8}\pm9.0$ \\
walk-m-r & medium & gravity & $36.5\pm5.3$ & $\textbf{44.2}\pm5.8$ & $43.2\pm4.5$ \\
hopp-m & medium & kinematic & $\textbf{65.0}\pm0.9$ & $60.8\pm5.5$ & $64.0\pm2.2$ \\
half-m & expert & kinematic & $\textbf{30.7}\pm9.8$ & $16.9\pm12.9$ & $15.6\pm4.9$ \\
half-m-e & medium & morph & $\textbf{43.6}\pm0.3$ & $42.9\pm1.1$ & $43.3\pm1.1$ \\
\bottomrule
\end{tabular}
\end{table*}

\begin{table*}[htb]
\centering
\caption{Training time analyses for TAF-DT on D4RL datasets.}
\label{tab:training_times}
\resizebox{\textwidth}{!}{
    \begin{tabular}{l|cccccc}
    \toprule
    Dataset & \begin{tabular}{@{}c@{}}Preproc. \\ Peak VRAM\end{tabular} & \begin{tabular}{@{}c@{}}Preproc. \\ Time\end{tabular} & \begin{tabular}{@{}c@{}}Training \\ Peak VRAM\end{tabular} & \begin{tabular}{@{}c@{}}Command Network \\ Time\end{tabular} & \begin{tabular}{@{}c@{}}VAE Policy \\ Time\end{tabular} & \begin{tabular}{@{}c@{}}Total Training \\ Time\end{tabular} \\
    \midrule
    hopp-m$\to$\text{medium} & 3638 MB & 97.8s & 2108 MB & 2,314.4s & 585.6s & 24,086.1s \\
    \text{half-m}$\to$\text{medium} & 10398 MB & 98.0s & 4410 MB & 2,536.8s & 827.6s & 24,122.03s \\
    \bottomrule
    \end{tabular}
}
\end{table*}

\subsection{Ablation on OT cost functions}

We ablate the cost function used in the optimal transport (OT) computation (Table~\ref{tab:ot_dist}), comparing cosine cost (our default, following OTDF \citep{Lyu2025OTDF}) with Euclidean and squared Euclidean variants. The results show moderate, task-dependent sensitivity to this choice. For gravity shifts, \texttt{ant-m}$\to$\texttt{medium} is largely indifferent between cosine and squared Euclidean ($60.8\pm13.6$ vs.\ $60.8\pm9.0$), while \texttt{walk-m-r}$\to$\texttt{medium} clearly prefers $\ell_2$-based costs (up to $44.2\pm5.8$) over cosine ($36.5\pm5.3$). For kinematic shifts, \texttt{hopp-m}$\to$\texttt{medium} strongly favors cosine ($65.0\pm0.9$) over Euclidean variants, and \texttt{half-m}$\to$\texttt{expert} also sees cosine dominate ($30.7\pm9.8$ vs.\ $16.9\pm12.9$ and $15.6\pm4.9$). On \texttt{half-m-e}$\to$\texttt{medium} (morphology), all three costs are nearly tied around $43$.

Overall, performance gaps are usually modest but clearly task- and shift-dependent, indicating that the learned embedding is fairly robust to cost formulation while still exhibiting geometry-specific preferences. We retain cosine cost as the default to stay consistent with OTDF and because it offers competitive or best performance on most tested tasks, while future work could explore adaptive or learned cost metrics tailored to specific shift types.

\section{Efficiency and Overhead}\label{sec:training-time}

We report wall-clock time statistics for TAF-DT on two representative settings: \texttt{hopp-m}$\to$\texttt{medium} under gravity shift and \texttt{half-m}$\to$\texttt{medium} under morphology shift. Table~\ref{tab:training_times} summarizes the main phases of the pipeline:

\begin{itemize}
    \item \textbf{Preprocessing phase}: precomputes MMD scores, reference OT-row costs, and Gibbs source weights (97.8--98.0 seconds per tested dataset).
    \item \textbf{Training phase}: includes command network optimization (2,314--2,536 seconds) and policy training (585--827 seconds), with total runtime 6--7 hours for 100k gradient updates.
\end{itemize}

The MMD gating and OT soft-weighting preprocessing accounts for \textbf{$\le 0.5\%$} of total training time (about 98 seconds versus roughly 24,000 seconds), so the extra preprocessing cost is negligible relative to end-to-end training. This is consistent with TAF-DT's design: the stitchability-aware gate--then--weight step adds little overhead compared with the full optimization budget.


\end{document}

%% file: math_commands.tex
\usepackage{amsmath,amsfonts,bm}

\def\eqref#1{equation~\ref{#1}}

\def\1{\bm{1}}

\DeclareMathAlphabet{\mathsfit}{\encodingdefault}{\sfdefault}{m}{sl}
\SetMathAlphabet{\mathsfit}{bold}{\encodingdefault}{\sfdefault}{bx}{n}

